\documentclass[12pt]{article}
\usepackage[margin=1.1in]{geometry}
\usepackage[utf8]{inputenc}

\usepackage{amsmath,amsthm}
\usepackage{amsfonts,amssymb,bm}
\usepackage{subfigure}
\usepackage{graphicx}
\usepackage{caption}
\usepackage{epstopdf}
\usepackage{url}
\usepackage{listings}
\usepackage{xcolor}
\usepackage{hyperref}
\usepackage[capitalize,noabbrev]{cleveref}

\usepackage{natbib}  
\bibpunct{(}{)}{;}{a}{,}{,}

\theoremstyle{plain}
\newtheorem{theorem}{Theorem}
\newtheorem{proposition}[theorem]{Proposition}
\newtheorem{lemma}[theorem]{Lemma}
\newtheorem{corollary}[theorem]{Corollary}
\newtheorem{assp}[theorem]{Assumption}
\theoremstyle{definition}
\newtheorem{definition}[theorem]{Definition}
\theoremstyle{remark}
\newtheorem{remark}[theorem]{Remark}

\usepackage{algorithm}
\usepackage{algorithmic}

\usepackage{appendix}
\usepackage{braket}
\usepackage{multirow}
\usepackage{stmaryrd} 

\usepackage{authblk}

\usepackage{lineno}
\usepackage{tikz}
\usetikzlibrary{shapes.multipart,calc,fit}
\usetikzlibrary{matrix}

\DeclareMathOperator*{\argmin}{arg\,min}

\DeclareMathOperator*{\trace}{tr}

\DeclareMathOperator*{\covv}{cov}
\DeclareMathOperator*{\varr}{var}
\DeclareMathOperator*{\diag}{diag}

\DeclareMathOperator{\prox}{prox}
\DeclareMathOperator{\expe}{\mathbb{E}}

\newcommand{\reals}{\mathbb{R}}
\newcommand{\fro}[1]{\|#1\|_{\mathrm{F}}}       
\newcommand{\maxn}[1][X]{\|#1\|_{\mathrm{max}}}

\newcommand{\twopartdef}[4]
{\left\{
    \begin{array}{ll}
      #1 & \mbox{if } #2 \\
      #3 & \mbox{ } #4
    \end{array}
  \right.
}

\newcommand{\dop}{\mathrm{D}}

\newcommand{\trs}[1][x]{{#1}^{\mathrm{T}}}
\newcommand{\mc}[1]{\mathcal{#1}}
\newcommand{\grp}{G}
\newcommand{\gre}{E}
\newcommand{\grv}{V}

\DeclareMathOperator{\pa}{\mathrm{PA}}

\newcommand{\dom}[1][d]{\reals^{#1\times #1}} 
 
\DeclareMathOperator{\degr}{deg}  
\DeclareMathOperator{\nnz}{nnz}
\def\X{\mbox{\bf X}}
\def\E{\mbox{\bf N}}
\newcommand{\signoise}{{\Omega}}
\newcommand{\spd}{\mathcal{S}^{d}}   
\newcommand{\offd}[1]{#1_{\mathrm{off}}}

\DeclareMathOperator*{\mora}{\mathcal{M}}
\DeclareMathOperator*{\suppo}{E}

\newcommand{\doms}[1][\Theta]{E_{#1}}

\newcommand{\symm}{\mathrm{Sym}}

\newcommand{\eg}{e.g.}
\newcommand{\ie}{i.e.}

\newcommand{\er}{\mathrm{ER}}
\newcommand{\sff}{\mathrm{SF}}
\newcommand{\gau}{\mathrm{Gaussian}}
\newcommand{\expp}{\mathrm{Exponential}}

\newcommand{\sset}[1][\Theta]{\mathcal{S}(#1)}
\newcommand{\ind}{\perp\!\!\!\!\perp}

\newcommand{\augl}{\tilde{L}^{\rho}}
\newcommand{\subc}{\doms} 
\newcommand{\gd}{\mathrm{grad}_{\subc}}
\newcommand{\iter}[2]{{#1}_{{#2}}}
\newcommand{\ttau}{\tau_{\mathrm{98}}} 

\newcommand{\ellzero}[1]{ \|#1\|_{\ell_0}} 
\newcommand{\ellone}[1]{\|#1\|_{\ell_1}} 
\newcommand{\Duu}{D_{\mathrm{u}}} 

\newcommand{\sem}{SEM} 
\newcommand{\lingam}{LiNGAM}
\newcommand{\notears}{NOTEARS}
\newcommand{\golem}{GOLEM}
\newcommand{\ogho}{GH18}

\newcommand{\ourmo}{LoRAM} 
\newcommand{\ours}{ICID}
\newcommand{\XX}{ICID}
\newcommand{\XY}{$\mathcal{O}$-ICID}
\def\OXX{{$\mathcal{O}$-ICID}}

\def\XSkew{\XX\ (skew)}

\usepackage{color} 

\newcommand{\rev}[1]{#1} 
\newcommand{\revn}[1]{#1}

\newcommand{\comm}[1]{}
\newcommand{\checkk}[1]{}
\newcommand{\bleu}[1]{{\color{blue} #1}}

\usepackage{multicol}

\usepackage[normalem]{ulem}
\newcommand\soutpars[1]{\let\helpcmd\sout\parhelp#1\par\relax\relax}
\long\def\parhelp#1\par#2\relax{%
  \helpcmd{#1}\ifx\relax#2\else\par\parhelp#2\relax\fi%
  }
\renewcommand{\comm}[1]{}
\renewcommand{\bleu}[1]{#1}
\renewcommand{\checkk}[1]{}

\title{Learning Large Causal Structures from Inverse Covariance Matrix via Sparse Matrix Decomposition}

\author[*]{Shuyu Dong}
\author[$\dag$]{Kento Uemura}
\author[$\dag$]{Akito Fujii}
\author[$\dag$]{Shuang Chang}
\author[$\dag$]{Yusuke~Koyanagi}
\author[$\dag$]{Koji Maruhashi}
\author[*]{Mich\`ele Sebag}

\affil[*]{\small\it LISN, INRIA, Universit\'e Paris-Saclay} 
\affil[ ]{\small first.last@inria.fr} 

\affil[$\dag$]{\small\it Fujitsu Laboratories Ltd.} 
\affil[ ]{\small first.last@fujitsu.com} 

\date{\today} 

\begin{document}
\maketitle

\begin{abstract}
Learning causal structures from observational data is a fundamental problem
facing important computational challenges when the number of variables is
large. In the context of linear structural equation models (SEMs), this paper
focuses on learning causal structures from the inverse covariance matrix. The
proposed method, \rev{called \XX\ for Independence-preserving Decomposition from
Inverse Covariance matrix}, is based on continuous optimization of a
matrix decomposition model that preserves the nonzero patterns of the inverse
covariance matrix. Through theoretical and empirical evidences, we show that
\XX\ efficiently identifies the sought directed acyclic graph (DAG) assuming
the knowledge of noise variances. Moreover, \XX\ is shown empirically to be
robust under bounded misspecification of noise variances in the case where the
noise variances are non-equal. The proposed method enjoys a low complexity, as reflected by its time efficiency in the experiments, and also enables a novel regularization scheme that yields highly accurate solutions on the Simulated fMRI data \citep{smith2011network} in comparison with state-of-the-art algorithms. 
\end{abstract}


\section{Introduction}
Discovering causal relations from observational data emerges as an important problem for artificial intelligence~\citep{Pearl2000,peters2017elements} with fundamental and practical motivations. 
One notable reason is that causal models support modes of reasoning, \eg, counterfactual reasoning and algorithmic recourse \citep{WshopAlgorithmicRecourseICML21}, that are otherwise out of reach by correlation-based machine learning, as shown by \cite{peters2016causal,arjovsky2019invariant,Sauer2021ICLR}.

The learning of causal structures from data, also referred to as causal discovery, faces challenges in both statistical and algebraic aspects, since one needs to uncover not only correlations from data but also the underlying causal directions. %
In addition to difficulties related to a restricted number $n$ of observational samples hindering the estimation process, 
learning a directed acyclic graph (DAG) is NP-hard~\citep{chickering1996learning} even in the large sample limit, as the search space of DAGs increases super-exponentially with respect to the number $d$ of variables.

\paragraph{Related work.} A usual strategy to overcome computational and other difficulties is to constraint the DAG search \citep{spirtes2000causation,meek1995causal,chickering2002learning,loh2014high,solus2021consistency}. 
In the linear structure equation model (SEM), \cite{loh2014high} show that the
moral graph of the DAG coincides with the support graph
of the inverse covariance matrix under a mild faithfulness
assumption. 
This result entails a reduced search of DAGs within the support of the
inverse covariance matrix with respect to the log-likelihood score function involving the diagonal of noise variances $\Omega$. 
Furthermore, when the true $\Omega$ is provided, the score function admits the true DAG as the unique minimum \citep{loh2014high}. 
Note however that even when $\Omega$ is known, the actual minimum of score function is unknown, requiring a thorough search.
  
Another strategy for causal discovery is to formulate a continuous optimization
problem based on a differentiable or subdifferentiable (\eg, $\ell_1$-penalty)
score
function~\citep{NEURIPS2018_e347c514,aragam2019globally,ng2020role,ng2021reliable,lopez2022large}. 
In this line of work, the optimization of the causal model, defined on the set
of real square matrices, is subject to a differentiable DAG
constraint~\citep{NEURIPS2018_e347c514}, or alternatively, a sparsity-promoting
constraint on the sought matrix~\citep{ng2020role}. The computational cost of these optimization approaches, different from combinatorial methods, depends mainly on the gradient computations which scale well enough; %
but the overall complexity may still be high due to the nonconvex optimization landscape (see \eg, \citep{NEURIPS2018_e347c514,aragam2019globally}). %

\paragraph{Contributions.}
In this paper, we present a matrix decomposition approach aimed at getting the best of both above strategies. %
In the linear SEM setting, the inverse covariance matrix $\Theta$ is related to the parameters $(B,\Omega)$ of the SEM as: 
\begin{align} %
    \label{eq:dec}
  \Theta = (I-B)\Omega^{-1}\trs[(I-B)] %
\end{align}
where $B$ is the sought causal structure and $\Omega$ is the diagonal matrix of noise variances. 
Unfortunately, the above decomposition of $\Theta$ is not unique. We show \bleu{that the sought DAG $B$ can be 
uniquely identified (among all feasible DAGs) through this decomposition given the true noise variance matrix $\Omega$.} While 
this result takes inspiration from \cite{loh2014high}, a main
difference is that we use the exact matrix equation \eqref{eq:dec} 
for computing eligible matrix decompositions. 

As the found solutions are \bleu{not necessarily} DAGs, the underlying
optimization is augmented with an $\ell_1$-penalty as a sparsity-promoting term, since sparsity is shown to be an effective constraint for learning DAGs under mild assumptions on the true causal structure~\citep{raskutti2018learning,aragam2019globally,ng2020role}. %
\revn{We show that, under \cite[Assumptions 1--3]{ng2020role}, the $\ell_0$~minimization subject to the matrix equality~\eqref{eq:dec}, achieves exact learning of linear SEMs with equal noise variances.}
\rev{To leverage continuous optimization, we formalize this sparse matrix decomposition problem, under the name of \XX---for Independence-preserving Decomposition from Inverse
Covariance---}as an equality-constrained $\ell_1$-minimization problem analogous to the Dantzig selector~\citep{candes2007dantzig}, and we tackle this problem by an augmented Lagrangian method (ALM). 
\revn{The computational efficiency of \XX\ is largely explained by our choice of restricting the nonzeros of the candidate matrix $B$ within the support of $\Theta$, 
while the feasibility of the sought DAG is shown to be preserved.} 
More precisely, we show that the cost of the gradient computation in \XX\ scales as $O(kd^2)$ for causal models with a bounded Markov blanket size $k$, notably improving over \notears~\citep{NEURIPS2018_e347c514} and \golem~\citep{ng2020role}. %

\bleu{Experiments (\cref{sec:exps}) are conducted to empirically assess the performance of \XX\ in causal discovery tasks compared to state-of-art methods, which show significant speedups at the expense of a moderate loss of accuracy in the linear SEM setting with equal variances (EV).} 
\rev{Performances of \XX\ in learning from oracle inverse covariances are also examined, first in the EV case, and then in a challenging non-equal variance (NV) case (\eg, \citep{ReisachSW21}) with misspecifications of noise variances.} 
\revn{Furthermore, we experiment with the Simulated fMRI dataset \citep{smith2011network}. A regularization term using skewness-based measures~\citep{hyvarinen2013pairwise} is proposed and applied to \XX. The results show that the regularized \XX\ attains highly accurate solutions, besides Direct\lingam~\citep{shimizu2011directlingam,hyvarinen2013pairwise}, and at the same time achieves considerable speedups over state-of-the-art methods. }
%

\section{Background} 
\label{sec:bg} 

\subsection{Definitions and notation} 
A graph $\grp:=(\grv,\gre)$ consists of a set of nodes $\grv$ and a set of edges $\gre\subset \grv\times \grv$. Unless specified otherwise, all graphs (respectively, edges) are directed. 
The binary adjacency matrix of a graph $\grp$ is such that its $(i,j)$-th entry is $1$ if and only if $(i,j)\in\gre$. Conversely, any matrix $B\in\dom$ determines an unique edge set through the nonzero entries, $\suppo(B):=\{(i,j): B_{ij} \neq 0\}$, which we also refer to as the support of $B$. From the definition of edge set $\suppo(B)$, we assume by convention that all graphs are ordered (according to the indexing of matrix $B$). 
The number of nonzeros of matrix $B$ is denoted as \rev{$\ellzero{B}$} or $\nnz(B)$ indifferently. 
Given a directed acyclic graph (DAG) $\grp$, the moralization of $\grp$ is defined as the undirected graph $\mora(G)$, with all directed nodes in $\grp$ made undirected, and new undirected edges $(i,j)$ created for all pairs $(i,j)$ parents of a same node~$k$.
An undirected graph is a {\it chordal graph} if every cycle of length
greater than three has a chord, \ie, a shortcut that triangulates the cycle. 
\rev{For simplicity, a matrix is called a DAG (respectively, chordal) by abuse if its
support graph is a DAG (resp. chordal).} 
The set of $d\times d$ symmetric positive definite matrices is denoted by %
$\spd_{++}$. The positive definiteness of a symmetric matrix $\Theta$ is denoted as $\Theta\succ 0$.

\subsection{Structural equation models}
\label{ssec:rel-loh14}
Structural equation models are defined on a set of random variables $\X=(X_1,\dots,X_d)$. 
A linear SEM $(B, \signoise)$ expresses the causal relations among the variables as: 
\begin{align}
    \label{eq:sem}
\X = \trs[B]\X + \E
\end{align} 
where matrix $B\in\dom$ is supported on a DAG $\grp$ %
and $\E=(\epsilon_1,\dots,\epsilon_d)$ is a vector of $d$ noise variables. 
Here, the variables $\X$ and $\E$ are assumed to be centered and that $\epsilon_{j}\ind X_i$ for all
$i\in\pa_j^{\grp}$ where $\pa_i^{\grp}$ is the set of parent nodes of $X_i$ in $\grp$. 
The joint distribution $P_{\footnotesize\X}$ satisfies the Markov property with respect to $\grp$, and its density function can be factorized as $p({\bf x}) = \Pi_{i=1}^d
p\big(x_i|x_{\pa_i^{\grp}}\big)$.  

Since $\X$ and $\E$ are centered and $\epsilon_{j}\ind X_i$ for all
$i\in\pa_j^{\grp}$, %
the covariance of $\E$ is a diagonal matrix 
$\signoise:=\diag(\omega_1^2, \ldots, \omega_d^2)$, where $\omega_i^2$ denotes
the variance of $\epsilon_i$. The linear SEM~\eqref{eq:sem} with the DAG $B$ entails 
that %
$\X = (I-B)^{-1}\E$ since $(I-B)$ is invertible.   
Consequently, the covariance matrix of $\X$ is 
$\covv(\X) = (I-B)^{-\mathrm{T}} \signoise (I-B)^{-1}$, and the inverse
covariance matrix (or the precision matrix) of $\X$ reads:   
\begin{align} 
    \label{eq:prec-b}
\Theta := \varphi(B,\signoise^{-1}) = (I-B) \signoise^{-1} \trs[(I-B)]. 
\end{align} 
The property above yields the following result. 
As \Cref{thm:loh14-thm2} shows, the nonzero patterns (or edge set
$\suppo(\Theta)$) of the inverse covariance matrix is in fact a subset of the
moralized graph of $B$. 

\begin{lemma}[{\cite{loh2014high}}]
    \label[lemma]{lemm:loh14-lemm1} 
Let $\X$ be a random variable following the SEM~\eqref{eq:sem} of $(B, \signoise)$. 
Then the coefficients of the inverse covariance matrix $\Theta$ of $\X$ are as
follows, for all $i$ and $j \neq i$ in %
$[d]$: 
$\Theta_{ij} = -\frac{B_{ij}}{\omega_{j}^{2}}  - \frac{B_{ji}}{\omega_{i}^{2}} + \sum_{\ell=1}^d \frac{B_{i\ell}B_{j\ell}}{\omega_{\ell}^{2}}$, 
and 
$\Theta_{ii} = \frac{1}{\omega_{i}^{2}}  + \sum_{\ell=1}^d \frac{B_{i\ell}^2}{\omega_{\ell}^{2}}$. 
\end{lemma}
\begin{theorem}[{\cite{loh2014high}}] 
    \label{thm:loh14-thm2}
    The inverse covariance matrix $\Theta$~\eqref{eq:prec-b} reflects the
    graph structure of the moralization $\mora(B)$ through inclusion: for
    $i\neq j$, $(i,j)$ is an edge in $\mora(B)$ if $\Theta_{ij} \neq 0$.
\end{theorem}

The converse of \Cref{thm:loh14-thm2} is not always true but is considered as a
mild assumption (\Cref{assp:loh14-assp1}), stating that any edge $(i,j)$ in $\Theta$ yields either a
directed edge, %
or the existence of a common child, between $i$ and $j$ in $B$. 
This condition is a type of faithfulness assumption~\cite{koller2009probabilistic,spirtes2000causation}. %

\begin{assp}[{\cite{loh2014high}}] 
    \label[assp]{assp:loh14-assp1}  
Let $\X$ be a random variable following the SEM of $(B, \signoise)$. The inverse covariance matrix $\Theta$~\eqref{eq:prec-b} of $\X$ satisfies, for all $i\neq j$, 
$ \Theta_{ij} = 0$
only if $B_{ji}=B_{ij} = 0$ and $B_{i\ell}B_{j\ell}=0$ for all $\ell$. 
\end{assp}

\section{Independence-preserving decomposition of the precision matrix}
\label{sec:main-res}

Taking inspiration from \cite{loh2014high,ghoshal2018learning}, 
we consider the causal discovery problem with the linear SEM in a two-step
framework: i) estimating the inverse covariance matrix $\Theta$~\eqref{eq:prec-b} of $\X$ from observational data (statistical part); ii) recovering a causal structure matrix $B$ from $\Theta$ (structural learning part). 
We study the second, structural learning part assuming $\Theta$ is given by an oracle or estimated from observational data.

\subsection{Matrix decomposition within a sparse support}
\label{ssec:main-res}

Given an inverse covariance matrix $\Theta$, the matrix equation~(Eq. \ref{eq:prec-b}) generally admits multiple DAG solutions, %
and the set of solutions becomes even larger without the DAG constraint on $B$. 
From the %
results in \Cref{sec:bg}, however, it is possible to reduce the vast solution set %
by considering a restriction on the candidate support graph. \Cref{assp:loh14-assp1} ensures that the non-zero pattern (or edge set) of $\Theta$ coincides with the moralization of $B$. This allows us to define a matrix decomposition model more specific than \eqref{eq:prec-b}, which we call a {\it support-constrained} decomposition.

\begin{definition}%
    \label[definition]{def:idecomp}
Let $\Theta\in\spd_{++}$ be a positive definite matrix. The set $\sset(\Theta)$ is defined as the set of pairs of matrices $(B,D)$, where $B\in\dom$ and $D\succ 0$ is a strictly positive diagonal
matrix, such that:
\begin{align}\label{def:set-idecomp}
\sset[\Theta]:=\big\{(B,D): \Theta=(I-B)D\trs[(I-B)], \diag(B)=0, ~D\succ 0, 
    \text{~and~} \suppo(B)\subset\suppo({\Theta})\big\}.  
\end{align}
\end{definition}
A pair $(B,D)$ realizes a {\em support-constrained} decomposition of $\Theta$ if it belongs to
$\sset[\Theta]$. 
The domain of $B$ corresponding to the support constraint is denoted as 
\begin{align}
    \label{def:subsp} 
    \doms:= \{ B \in \dom: \diag(B)=\bm{0}, ~ \suppo(B)\subset \suppo(\Theta) \}. 
\end{align} 
The support-constrained decomposition is always well-defined in
the case %
of chordal matrices. In fact, the set $\sset[\Theta]$ is nonempty when $\Theta\succ 0$ is supported on a chordal graph; see \Cref{prop:main-1} (in \Cref{app:prfs-31}). 

In the general case where $\Theta$ is not a chordal matrix, then  $\sset$ contains the sought solution under \Cref{assp:loh14-assp1}: %
\begin{proposition} 
    \label[proposition]{prop:thm4} 
Let $\Theta$ be the inverse covariance matrix of $\X$ obeying the linear SEM with $(B,\signoise)$, \ie, $\Theta=\varphi(B,\signoise)$~\eqref{eq:prec-b}. 
Suppose that this SEM satisfies \Cref{assp:loh14-assp1}, then the set $\sset[\Theta]$ contains $(B,\Omega^{-1})$. 
\end{proposition}  
The proof to this proposition, and the proof of all results in this section are given in \Cref{app:prfs}. 

\rev{In addition to the sought solution $(B,\Omega)$, 
$\sset[\Theta]$ also contains pairs ($B',\Omega')$ for 
$B'$ Markov equivalent to $B$. \cite{ghassami2020characterizing} shows the specific rotations $Q$ that realize the transformation from $(I-B)$ to $(I-B')$ and $\Omega$ to $\Omega'\succ 0$ such that $\varphi(B',\Omega')=\varphi(B,\Omega)$. %
Moreover, for any $B'$ Markov equivalent to $B$, $\mora(B)=\mora(B')$ which coincides with $\suppo(\Theta)$. Hence $B'$ also satisfies the support constraint regarding the set \eqref{def:set-idecomp}.} 
In particular, we show next that any DAG $B$ such that $(B,\Omega^{-1})$ belongs to the set~\eqref{def:set-idecomp} with the true noise variances matrix $\Omega$, is the sought true causal structure: 

\begin{theorem}
    \label{prop:main-2}
Let $(B^\star,\signoise^\star)$ be a linear SEM satisfying \Cref{assp:loh14-assp1} and let
$\Theta=\varphi(B^\star,\signoise^\star)$ be the inverse covariance of $\X$ following this SEM. 
If a DAG $B \in \dom$ is such that $(B,{\Omega^\star}^{-1}) \in \sset[{\Theta}]$, then $B=B^\star$.
\end{theorem}
\rev{The result of \Cref{prop:main-2}, similar to \cite[Theorem 7]{loh2014high}, requires the knowledge of the noise variances.  
\cite[Theorem 7]{loh2014high} supported causal discovery through a DAG enumeration approach. 
In our case, \Cref{prop:main-2} supports a causal discovery approach based on a novel sparse matrix decomposition method. 
}

\subsection{\XX\ model learning via constrained optimization}

\revn{\Cref{prop:main-2} suggests searching for a DAG that realizes the support-preserving matrix decomposition of $\Theta$. 
One way to tackle this challenging task is by relaxing the DAG constraint %
and minimizing the 
number of edges (nonzeros of $B$) subject to the constraint that $\varphi(B,D) = \Theta$:
    \begin{align} 
        \underset{B\in \doms,~ D\succ 0 }{\text{minimize}} ~ \ellzero{B} \quad \text{subject to~~} (I-B)D\trs[(I-B)]=\Theta 
            \label{prog:oicid-l0}
    \end{align}   
where the search space $\doms$ of $B$ %
is defined in \eqref{def:subsp}. %
Noting that the constrained set \eqref{def:set-idecomp} contains at least one DAG
(\Cref{prop:thm4}), the $\ell_0$ minimization 
helps the selection by discarding graphs with cycles, 
since cycles usually implies redundant edges from spurious causal directions. 
The following corollary, deduced from \Cref{prop:main-2} and \cite[Theorem 1]{ng2020role}, 
supports this claim in the case with equal noise variances (EV case) under 
a generalized faithfulness assumption and a causal minimality condition~\citep{ghassami2020characterizing,ng2020role}.  
\begin{corollary}
    \label[corollary]{coro:main-2}
Let $\X$ be a variable following a linear SEM $(B^\star,\Omega^\star)$. Suppose the precision matrix $\Theta$ of $\X$ satisfies the Assumptions 1--2 of \cite{ng2020role}, and that 
the noise variances are equal. Then, %
as long as the skeleton of $B^\star$ does not contain any triangle, the $\ell_0$ minimization~\eqref{prog:oicid-l0} has a unique global optimum $(B,D)$ where $D_{ii} \equiv \min_{j\in[d]}\{\Theta_{jj}\}$ and $B=B^\star$. 
\end{corollary}
}

\revn{To benefit from the $\ell_0$ minimization effects %
while leveraging continuous optimization techniques for matrix decomposition, we resort to the following matrix decomposition problem:
    \begin{align} 
        \underset{B\in \doms,~ D\succ 0 }{\text{minimize}} ~ \ellone{B} \quad \text{subject to~~} 
(I-B)D\trs[(I-B)]=\Theta \label{prog:oicid+eb}
    \end{align}  
where 
$\ellone{B} := \sum_{i,j} |B_{ij}|$ is a continuous relaxation of $\ell_0(B)$  for limiting the number of nonzeros in the solution, and the search space $\doms$ of $B$ is defined in Eq. \ref{def:subsp}.
}

\rev{The $\ell_1$ minimization problem~\eqref{prog:oicid+eb}} %
is analogous to the Dantzig selector~\citep{candes2007dantzig} for high-dimensional statistical problems. 
\rev{Combined with an inverse covariance estimator, this problem~\eqref{prog:oicid+eb} yields 
a method for causal discovery from observational data.} 
\rev{In the context of causal structure learning, we use the term {\it
independence-preserving decomposition} (\XX) %
to refer to %
problem~\eqref{prog:oicid+eb}. 
Note that the locations of the zeros of $\Theta$ encode the independence
relations within $\X$ (in the Gaussian setting, \cite[Remark 3]{loh2014high}),
thus the preservation of the independent pairs after the decomposition, \ie, $\{(i,j):
\Theta_{ij}=0\} \subset \{(i,j): B_{ij}=0\}$, corresponds exactly to the
support constraint $\suppo(B)\subset\suppo(\Theta)$ of the problem
\eqref{prog:oicid+eb}. 
}

\revn{
\paragraph{Additional regularization.}
The \XX\ problem~\eqref{prog:oicid+eb} is deduced from the linear SEM without
restriction on the distributions of the {\bf X} and {\bf N} variables
(\Cref{ssec:rel-loh14}). 
Therefore the \XX\ approach is open to the possibility of additional
regularization schemes depending on the specific properties of the variable and noise distributions. In cases where the data distributions
are skewed, for example, the third-order cumulant statistics carry
important information of the pairwise causal
directions~\citep{hyvarinen2013pairwise}. Taking this inspiration,
we introduce a skewness-based regularization for \XX\ as follows,
\begin{align} 
    & \underset{B\in \doms,~ D\succ 0 }{\text{minimize}} ~ \ellone{B} + \lambda_2 \trace(\trs[\tilde{M}](B\odot B)) \quad \text{subject to~~} (I-B)D\trs[(I-B)]=\Theta, \label{prog:icid-reg} \\ 
    & \hspace{10mm} \text{where ~ ~ } 
    \tilde{M} := - C_{\footnotesize\X} \odot \expe\big[\X\trs[g(\X)] - g(\X)\trs[\X]\big]. \label{eq:reg-matm} 
\end{align}
Here $C_{\footnotesize\X}$ denotes the correlation matrix of $\X$ and $g(\cdot)$
applies to $\X$ elementwisely. For skewed distributions, we take $g(x)= -x^2$ in the
spirit of \cite[Theorem~2]{hyvarinen2013pairwise}; details are given in \Cref{ssec-app:icid-reg}. 

The algorithms described in the next subsection are meant to solve problem \eqref{prog:oicid+eb} and its regularized formulation~\eqref{prog:icid-reg}.
}

\subsection{Algorithm} 
\label{sec:alg-proposed} \label{sec:oicid}
\label{sec:algo}

\revn{
As the \XX\ problems~\eqref{prog:oicid+eb} and~\eqref{prog:icid-reg} are nonconvex and nonsmooth 
(since the feasible set of the matrix equation $\varphi(B,D)=\Theta$ is
nonconvex and the $\ell_1$-term in the objective is nonsmooth), 
we consider the augmented Lagrangian method (ALM)~\citep{Bertsekas1999} to
solve them. %
}
\Cref{alg:oicid} presents the procedure for optimizing the augmented
Lagrangian of \XX. The presented algorithm assumes the equal variance setting ($\Omega$ is a multiple of the identity matrix). Under this assumption, the sought solution $B$ can be retrieved in light of \Cref{coro:main-2}.\footnote{It will be shown experimentally on real-world data that the algorithm can also yield good performances in the general case,
when the equal noise variance assumption is not guaranteed.}

The equality constraint of \XX\ consists of $\frac{d(d+1)}{2}$
equalities. \rev{Let $\Lambda$ be the $d\times d$ matrix of the Lagrange multipliers
for these equalities. Then} the augmented Lagrangian of \XX\ given $D$ is
    \begin{align} 
        L^{\rho} (B; \Lambda) = f(B) + \braket{\Lambda,\Theta-\varphi(B,D)} + \frac{\rho}{2} \fro{\Theta-\varphi(B,D)}^2
        \label{eq:lagr-oicid+e}
    \end{align}
where $f(B) := \ellone{B}$ for problem~\eqref{prog:oicid+eb} and $f(B):=\ellone{B}+r(B)$ 
with $r(B) = \lambda_2 \trace(\trs[\tilde{M}](B\odot B))$ for problem \eqref{prog:icid-reg}.

\begin{algorithm}[htpb]
    \caption{ALM for \XX\ with a fixed diagonal $D$ \label{alg:oicid}}
    \begin{algorithmic}[1]
        \REQUIRE{Inverse covariance matrix $\Theta$, parameters
        $\beta\in(0,1)$, $\rho_0>0$, tolerance $\epsilon>0$} 
        \ENSURE{$B_t\in\doms$}
        \STATE {\bf Initialize}: $B_0\leftarrow \bm{0}$, $\Lambda_0 \leftarrow\bm{0}$, $\rho\leftarrow\rho_0$,
        \revn{and $D_{ii} \leftarrow\min\limits_{j\in[d]}\{\Theta_{jj}\}$ for all $i\in[d]$. 
        }
        \label{algl:init}
\FOR{$t=1,\dots,$}
\STATE \label{algl:primal}
{\bf Primal descent}: for $L^{\rho}(B;\Lambda)$ defined in
\eqref{eq:lagr-oicid+e}, compute 
\hfill \revn{\it \texttt{\#} see \Cref{alg:id-fista} (\Cref{sec-app:fista})} 
\begin{align}\label{prog:primal-oicid}
B_t = \argmin_{B\in \doms} L^{\rho} (B; \Lambda_{t-1})  
\end{align} 
with $\rho \geq \rho_0$ such that $\mc{I}_t:=\fro{\Theta-\varphi(B_t,D)} \leq \beta \mc{I}_{t-1}$ %
        
\STATE {\bf Dual ascent}: \hspace{23mm} 
                $\Lambda_{t} \leftarrow \Lambda_{t-1} + \rho (\Theta-\varphi(B_t,D))$  
\STATE {\bf if} {$\fro{\Theta-\varphi(B_t,D)}\leq \epsilon$} {\bf then}
                \quad return $B_t$ as solution 
\ENDFOR
    \end{algorithmic}
\end{algorithm}

\rev{The efficiency of  \Cref{alg:oicid} mostly depends on the resolution of the primal descent problem (line~\ref{algl:primal}; Eq. \eqref{prog:primal-oicid}). 
For this reason, an adaptation of the FISTA~\citep{beck2009fast} is designed (\Cref{alg:id-fista}, \Cref{sec-app:fista}) according to the structure of the primal problem~\eqref{prog:primal-oicid}. 
The adaptation is in the computation of the gradient
descent stepsize, which is based on an efficient second-order approximation of the line minimization of (the smooth part of) $L^{\rho} (B; \Lambda_{t-1})$. 
Details are given in \Cref{parag:stepsize} (see Eq. \eqref{eq:def-alpha0}).
}

\paragraph{Computational cost.}

The dominant cost of \Cref{alg:id-fista} corresponds to the computation of the gradient $\nabla_B \augl(\cdot,\Lambda)$ at each iteration, where $\augl(\cdot,\Lambda)$ denotes the smooth part of $L^{\rho}(\cdot,\Lambda)$.  
As \Cref{prop:hess-oicid} (\Cref{sec-app:alg}) shows, given the support constraint of $B$, this gradient computation only requires $O(kd^2)$ floating-point operations, where $k$ is the maximal node degree of the graph of $\Theta$ (which equals the maximal Markov blanket size of the true DAG under
\Cref{assp:loh14-assp1}).

This contributes to the efficiency of the proposed algorithm. 
Indeed, as long as the maximal Markov blanket size of the true DAG is upper-bounded by a constant $k$ independent of the graph size $d$, then 
the per-iteration cost of \Cref{alg:id-fista} is bounded by $O(d^2)$, \revn{which is greatly reduced compared to other continuous optimization-based approaches.}

\paragraph{Discussion.}
It is worth noting that \XX\ does not rely on any estimation of a causal ordering for learning matrix $B$, 
as opposed to ordering-based methods \citep{ghoshal2018learning,chen2019causal,gao2022optimal}. 
However, a causal ordering estimated from a solution $B$ of \XX\ can be used in an ordering-based matrix decomposition approach, using \eg, Cholesky decomposition of~$\Theta$ or \citep[Algorithm 1]{ghoshal2018learning}, 
for yielding potentially improved candidate solutions. 
Such a combination of \XX\ with ordering-based matrix decomposition approaches is left for further work.

\section{Experiments} 
\label{sec:exps}\label{sec:resul}

\def\XX{{ICID}}
\def\OXX{{$\mathcal{O}$-\XX}}

We conduct experiments on synthetic and real-world data to assess the performance of \XX\ in causal structure learning from inverse covariance matrices and also in causal discovery tasks. The primary goal of the experiments is to examine the learning accuracy of the proposed method and its computational efficiency in different settings.

The proposed algorithms are available at \url{https://github.com/shuyu-d/icid-v2}.

\subsection{Experimental setting}

\paragraph{Benchmark data.}
Following~\cite{NEURIPS2018_e347c514}, the synthetic data are generated from linear SEMs where the causal structure $B$ is drawn from the Erd\H{o}s--R\'enyi (ER) or the scale-free (SF)~\citep{barabasi1999emergence} random graphs. The coefficients (edge weights) of $B$ are drawn from the uniform distribution $\text{Unif}([-2,-0.5]\cup [0.5, 2])$. 
Experiments are also conducted on the Simulated fMRI dataset~\citep{smith2011network}.

\paragraph{Baselines.}
The \XX\ method is tested in two scenarios. 
One scenario is to assess the proposed causal learning methodology irrespective of the statistical issues of inverse covariance estimation, in which case \XX\ is exceptionally labeled as \XY, meaning that the precision matrix $\Theta$ is the ground truth one. 
The other scenario is causal discovery, in which case \XX\ takes as input a precision matrix learned from observational data using a sparse inverse covariance estimator (details in \Cref{alg:ice-emp}, \Cref{ssec:app-sel}). 

In the benchmarks, 
\XY\ is tested along with~\cite[Alg.~1]{ghoshal2018learning}---labeled as \ogho---which is also given the ground truth $\Theta$ as input. Both methods are compared with GES~\citep{chickering2002optimal} and its Java-based implementation fastGES~\citep{ramsey2017million}, which take the observational data as input. \XX\ is compared with Direct\lingam~\citep{shimizu2011directlingam},
ICA-\lingam~\citep{shimizu2006linear}, GES/fastGES, \notears~\citep{NEURIPS2018_e347c514} and \golem~\citep{ng2020role}. 
The estimated graphs are evaluated by usual metrics (SHD, TPR, FDR and FPR) (details in \Cref{ssec-app:metrics}).

The CPU-based methods are run on one CPU of Intel(R) Xeon(R) Gold 5120 14 cores @ 2.2GHz; GOLEM is run on a GPU of Tesla V100-PCIE-32GB.

\subsection{Scalability in causal discovery}
\label{ssec:exp-scala}

In this experiment, \XX\ is evaluated in causal discovery tasks, in comparison with \notears, \golem\ and GES / FastGES, on random DAGs within the ER1 set with equal variance noise. The number $d$ of nodes varies in $\{$100, 200, 400, 600, 800, 1000, 2000$\}$, and the number of samples of $\X$ is set as $n=10 d$. The performances of \XY\ (provided with the ground truth $\Theta$) are also reported to showcase the structural learning aspect of \XX.

The optimization parameters of \XX, selected after a simple tuning using grid search, are given in \Cref{parag-app:param-opt}. In this benchmark, \XX\ uses a single ALM iteration (\Cref{alg:oicid}) followed by a projection on the DAG space for eliminating cycles (details in \Cref{sec-app:icid}). 
\revn{\XY\ uses multiple ALM iterations with respect to an accuracy-based stopping criterion (details in \Cref{sec-app:fista}).} 
\XX\ and \notears\ use the same CPU resource, similar to fastGES, and \golem\ is run
on GPU. %

\begin{figure}[htpb]
    \centering
    {\includegraphics[width=.98\textwidth]{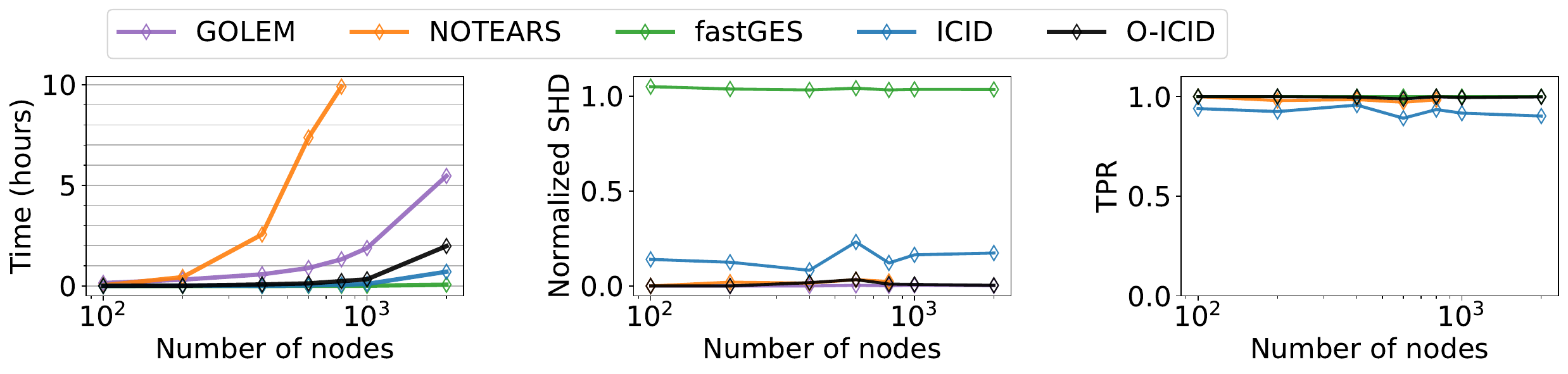}}  
\caption{Causal discovery results vs number of nodes. The number of nodes $d$ range from 100 to 2000. The SHD scores in the middle subplot are normalized by $\text{nnz}(B^\star)$. %
}
\label{fig:exp1}
\end{figure}

The results in \Cref{fig:exp1} show that \XX\ scales up efficiently---with a 5 times speedup over \golem\ for $d=2000$; and even greater speedups over \notears\ for $d \geq 600$), at the expense of a moderate loss in terms of TPR (above 80\%) %
and the normalized SHD (below 20\%). 
\XX\ outperforms fastGES in terms of SHD, within a comparably low computation time in all settings of~$d$. 
\XY\ yields close-to-exact learning accuracy in both TPR and SHD---besides \golem---while demonstrating good scalability.

\subsection{Learning from the precision matrices and robustness}
\label{ssec:exp-ev-nv}

In this experiment, we further examine the structural learning part of \XX~(\Cref{alg:oicid}) given oracle precision matrices, namely \XY. 
The linear SEMs are generated for the EV case and a NV case described as follows. %

Given a linear SEM variable $\X$ with equal noise variances, let $\Duu$ be the diagonal matrix such that $(\Duu)_{ii} = \sqrt{\varr(X_i)}$, 
and consider the following family of SEMs:
\begin{align} \label{eq:def-xlam}
    \X^{(\lambda)} := \Duu^{-\lambda} \X \quad \text{for~ ~} 0 < \lambda \leq 1.
\end{align}
It can be shown that $\X^{(\lambda)}  = \trs[(\Duu^{-\lambda} B \Duu^{\lambda})] \X^{(\lambda)} + \Duu^{-\lambda} \E$.
The noise variance matrix of $\X^{(\lambda)}$ is $\tilde{\E} = \Duu^{-\lambda}\E$, which has a diagonal covariance matrix 
$\tilde{\Omega} = \Duu^{-\lambda} \Omega \Duu^{-\lambda} = \Duu^{-2\lambda}$; note that  $\X^{(1)}$ is the standardization of $\X$.

\revn{In the first study, in light of \Cref{prop:main-2}, we evaluate the effectiveness of \XY\ in the NV ($\lambda>0$) case, given the true noise variances.}  
Hence \XY\ is tested using the oracle diagonal matrix $D^*:=\tilde{\Omega}^{-1}$ (overriding $D$ in \Cref{alg:oicid}, line 1) when $\lambda>0$. 
The EV ($\lambda=0$) case is also tested as a reference. 
The baseline method \ogho\ is tested using the same inverse covariance matrices, and GES (and/or fastGES) are tested using the observational data. 
For each setting of $\lambda \in \{$0, 0.1, 0.2, 0.4, 0.8, 1.0$\}$, the methods are evaluated with 10 random linear SEMs on ER$k$ DAGs ($k=$ 1, 2, 3). 
The number of nodes is $d=$ 50. 

\begin{figure}[htpb]
    \centering
    \subfigure[ER1]{\includegraphics[width=0.73\textwidth]{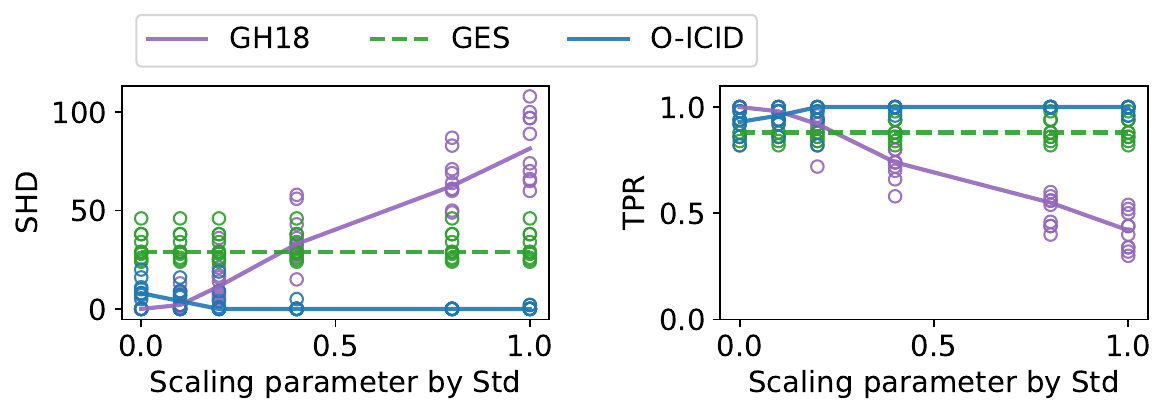}}
    \subfigure[ER3]{\includegraphics[width=0.73\textwidth]{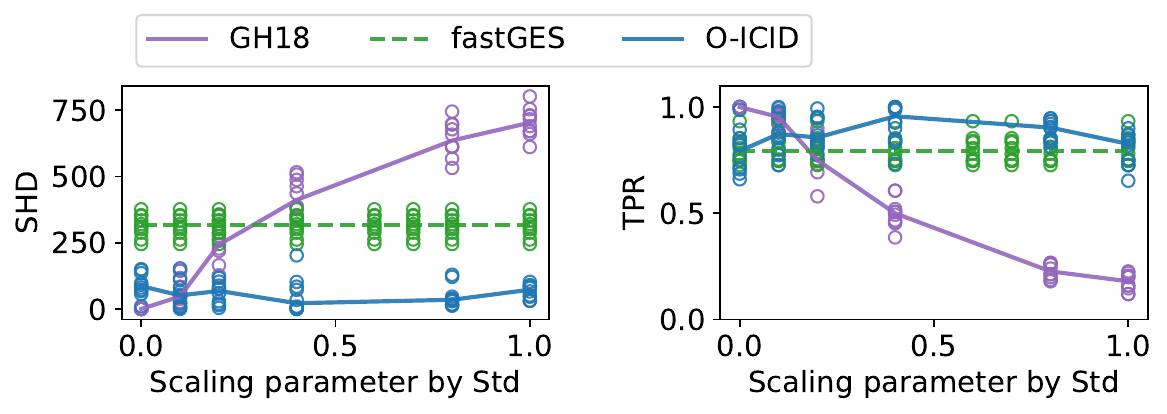}} 
    \\
    \subfigure[ER1, $\lambda=0.4$]{\includegraphics[width=0.73\textwidth]{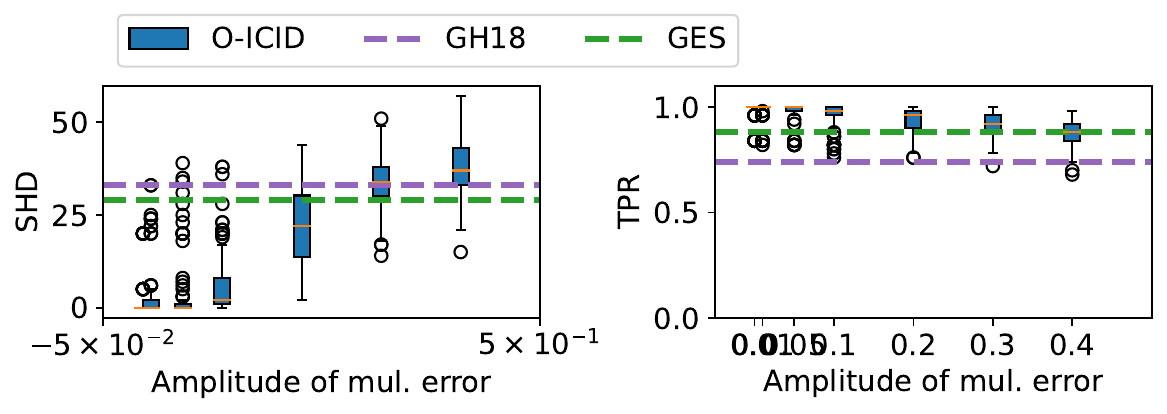}} 
    \subfigure[ER1, $\lambda=1$]{  \includegraphics[width=0.73\textwidth]{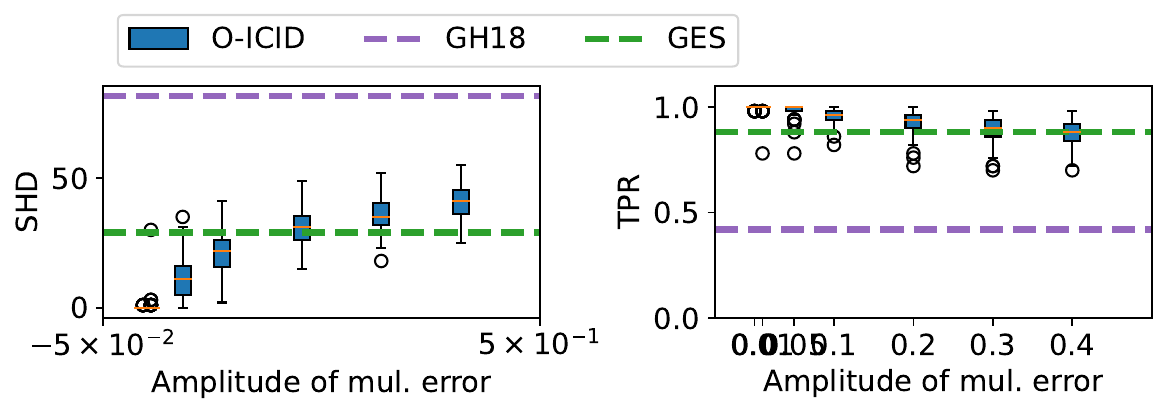}} 
    \caption{Causal structure learning by \XY\ compared to \ogho\ and GES.  The true DAGs are drawn from the
    ER1 and ER3 sets with $d=$ 50 nodes. 
    In (a)--(b): the x-axis indicates the scaling parameter $\lambda$ in~\eqref{eq:def-xlam}.
    In (c)--(d): the SHD plots use the log scale for the x-axis of ${\sigma}$ while the TPR plots use the linear scale for the same values of ${\sigma}$. }
    \label{fig:oicid-std-interp-flexD-v-other2-upd} 
\end{figure}

The results in \Cref{fig:oicid-std-interp-flexD-v-other2-upd} (a)--(b) (on ER1 and ER3) 
show that \XY\ and \ogho\ reach close-to-exact or exact learning accuracy (with often zero SHDs) %
for very small values of $\lambda$, corresponding to the EV case.
When $\lambda$ increases, however, the baselines of \ogho\ are seriously degraded for $\lambda\geq$~0.4. 
\revn{We notice that this is due to the fixed, variance-based node ordering rule of \ogho, which is tailored for the EV case (and cases with small $\lambda$); see \cite[Assumption 1]{ghoshal2018learning}.} 
By contrast, \XY\ (with the true noise variances given) remains stable and almost insensitive to $\lambda$. 
\revn{Such stability is observed on all ER$k$ graphs tested. The stability of \XY\ on SF2 and SF4 graphs is satisfactory in the sense that it has better or comparable accuracy than the baselines of GES for all $\lambda$.} 
The complete set of results are detailed in \Cref{ssec-app:exp-ev-nv}. 
Theses observations suggest the effectiveness of the proposed \Cref{alg:oicid} for \XX\ given the true noise variances. 

The computation time for \XY\ is around 1 to 3 seconds (Tables~\ref{tab:std-interp} and \ref{tab:std-interp-er3}, in \Cref{ssec-app:exp-ev-nv}) on the ER$k$ graphs, which is comparable to fastGES (for ER3) and slightly longer than \ogho, respectively, and is 
much shorter than GES, as could be expected from their theoretical complexities.

\paragraph{Misspecification of $D$.} 
The second study considers the NV case where \Cref{alg:oicid} is provided with a corrupted noise variance matrix $\hat D$. The misspecification of $\hat{D}$ is parametrized by multiplicative error terms: for all $i\in [d]$ independently, 
$\hat{D}_{ii} := \max(\frac{1}{2}, 1+ {\epsilon}) \tilde{\Omega}_{ii}^{-1}$, with ${\epsilon} \sim \mathcal{N}(0,{\sigma})$ for ${\sigma} \in (0, \frac{1}{2})$, 
where the minimal value $\frac{1}{2}$ in the multiplicative term is a safeguard
to ensure positive definiteness of~$\hat{D}$. 
The results are reported for ${\sigma}$ in $\{$0, 0.01, 0.05, 0.1, 0.2, 0.3, 0.4$\}$, in \Cref{fig:oicid-std-interp-flexD-v-other2-upd} (c)--(d).
\revn{For ER1, $\lambda=$ 0.4, a relatively challenging NV case (in view of \Cref{fig:oicid-std-interp-flexD-v-other2-upd} (a)),} the comparisons in \Cref{fig:oicid-std-interp-flexD-v-other2-upd} (c)--(d) show that \XY\ outperforms the baselines (the dashed lines marking the average scores) of \ogho\ and GES within moderate multiplicative misspecifications (${\sigma} \leq 0.2$) in terms of SHD and TPR. %
\revn{Similar comparisons (Figures~\ref{fig:oicid-std-interp-Dpert-2} and \ref{fig:oicid-std-interp-Dpert-3}, in \Cref{ssec-app:exp-ev-nv}) are observed with other settings on ER$k$ graphs.}

\subsection{Experiments on Simulated fMRI Data} 
\label{sec:fmri}

The simulations of the fMRI time series are described in detail by \cite{smith2011network}. %
Networks of varied complexity were used to simulate fMRI timeseries. The simulations were based upon the dynamic causal modelling (DCM) forward model~\citep{friston2003dynamic}. %
DCM uses the nonlinear balloon model~\citep{buxton1998dynamics} %
for the vascular dynamics, that is, the connection between the neural activities and the measured signal, sitting above a simple neural network model of the neural dynamics. 
The nodes of the DCM corresponded to brain regions. The external binary inputs (encoding neural activations) to the nodes are generated, which are not the same as the noise variables in the SEM, although related. %
Neural noise of standard deviation 0.05 of the difference in height between the two states was added.  
Different simulation datasets (Sim1--Sim28, \cite[Table~1]{smith2011network}) were generated depending on parameters including the number of nodes, duration of the fMRI session, and the number of separate realizations (``subjects''). 

The causal discovery methods are examined on Sim3 ($d=$ 15 nodes) and Sim4 ($d=$ 50), respectively, using all subjects' time series as observational data. 
The regularized \XX\ problem~\eqref{prog:icid-reg}, labeled as \XSkew, is motivated by   
the prior knowledge on the fMRI data that its skewness is mainly positive \citep[\S2.9]{hyvarinen2013pairwise}. 
Specifically, the matrix $\tilde{M}$~\eqref{eq:reg-matm} in the regularization term of \XSkew\ is defined with a third-order cumulant-based pairwise measure \cite[Theorem~2]{hyvarinen2013pairwise}, and the choice of $g(x):=-x^2$ in \eqref{eq:reg-matm} reflects the prior knowledge about the skewness of the fMRI data. 
More details including the Sim3 and Sim4 datasets, and the parameter selection of \XSkew\ are given in \Cref{ssec-app:icid-reg}. 

        \begin{table}[htpb]
        \small 
        \centering
        \caption{Causal discovery results on the fMRI Sim4 data~\citep{smith2011network}.
            \label{tab:exp-fmri-sim4}}
        \begin{tabular}{c|ccccc}
        \hline
        {\it fMRI Sim4}               & SHD  & TPR     & FDR   & FPR     & time (seconds) \\
        \hline                                   
        DirectLiNGAM     & 56   & 1.0    & 0.479 & 0.048   & 168.722     \\ 
        ICA-LiNGAM       & 158  & 0.754  & 0.775 & 0.136   & 10.307     \\ 
        GES              & 90   & 0.836  & 0.641 & 0.078   & 709.984    \\ 
        fastGES          & 89   & 1.0    & 0.657 & 0.101   & 1.121      \\ 
        \XX              & 14   & 0.820  & 0.219 & 0.012   & 0.679      \\ 
        \XX\ (skew)      & {\bf 6}   & 0.984  & {\bf 0.090} & {\bf 0.005}   & {\bf 0.517}      \\ 
        \hline
        \end{tabular}
        \end{table}

\begin{figure}[htpb]
  \centering
  \begin{minipage}[b]{0.31\textwidth}
  \centering
    \includegraphics[width=\textwidth]{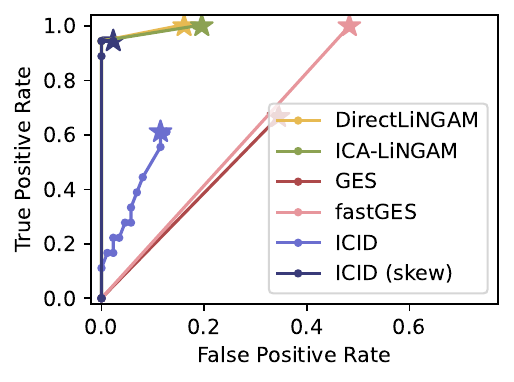}
    \captionof*{figure}{(a) fMRI Sim3}
  \end{minipage}
  \hspace{3mm}
  \begin{minipage}[b]{0.31\textwidth}
  \centering
    \includegraphics[width=\textwidth]{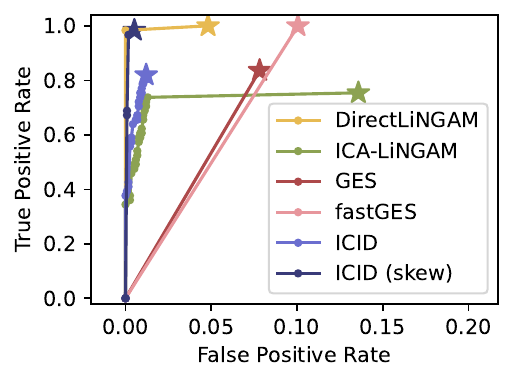}
    \captionof*{figure}{(b) fMRI Sim4}
  \end{minipage}
  \hfill
  \begin{minipage}[b]{0.32\textwidth}
    \centering
    \includegraphics[width=\textwidth]{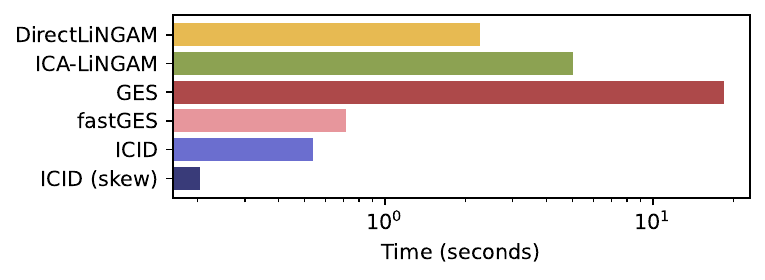} \\ 
    \includegraphics[width=\textwidth]{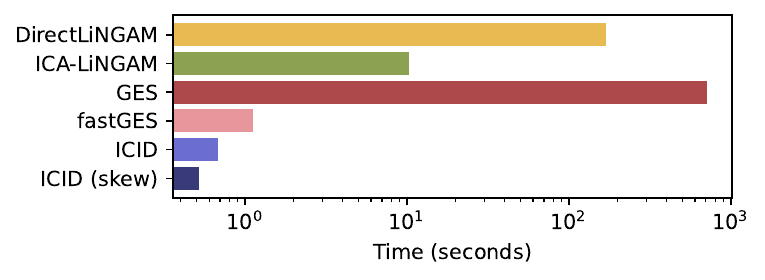} 
    \captionof*{figure}{(c) Runing time}
  \end{minipage}
  \caption{Results on the Simulated fMRI datasets (Sim3 and Sim4): (a)--(b) ROC curves. (c)~Running time on Sim3 (upper) and Sim4 (lower).\label{fig:fmri-all}}
\end{figure}

The performances of \XX\ and \XSkew\ on the Sim4 data, together with the baselines of Direct\lingam, ICA-\lingam, GES and fastGES, are reported in \Cref{tab:exp-fmri-sim4}. 
Moreover, on Sim3 and Sim4, each method is evaluated through the ROC curves of their estimated graphs, given in \Cref{fig:fmri-all} (a)--(b), together with the running time in \Cref{fig:fmri-all} (c). 
The results show that \XX\ obtains the second best SHD score (after \XSkew) on Sim4, with a good trade-off between the TPR and the FDR, and a computation time lower than fastGES. Moreover, on Sim3 and Sim4, \XSkew\ recovers almost exactly the true DAGs (attaining the lowest SHDs and close-to-one TPRs)---besides Direct\lingam---with a significantly low computation time.

\section{Discussion}
\label{ssec:relatedwork} 
The two approaches most related to \XX\ are that of
\cite{loh2014high} and \cite{ghoshal2018learning} %
as all rely on the inverse covariance matrix, yielding an identifiability result provided the knowledge of the noise variances.
The main difference regards the algorithmic approaches. 
\cite{loh2014high} proceed by defining a likelihood-based score on DAG structures, and then by searching for its minimizer using dynamic programming, assuming that the %
support graph of $\Theta$ %
has a restricted tree-width. 
\cite{ghoshal2018learning} proceed iteratively, recovering a topological order that is consistent with the causal structure. At each iteration a terminal variable 
is determined and removed from the linear SEM, while the corresponding row of the causal DAG is deduced using~\eqref{eq:prec-b}.

The difference between \XX\ and \citep{loh2014high} is twofold. On one hand, \XX\ tackles a decomposition of the inverse covariance matrix subject to a support constraint. 
\Cref{alg:oicid} thus addresses a constrained continuous optimization problem~\eqref{prog:oicid+eb}; an adapted FISTA is used in the ALM of \Cref{alg:oicid}, \revn{which has a significantly reduced computational cost due to the support constraint in \eqref{prog:oicid+eb}.} 
On the other hand, \XX\ enables to check the optimality of the provided solution (via the matrix equality constraint) while the approach in \citep{loh2014high} requires an exhaustive enumeration of the candidate solutions. 
\XX\ offers the same identifiability guarantees as \cite[Theorem 7]{loh2014high} under the same assumptions, as shown in \Cref{prop:main-2} and supported empirically in \Cref{ssec:exp-ev-nv}. %

Likewise, the difference between \XX\ and \citep{ghoshal2018learning} lies in the algorithmic strategies. 
The complexity of \ogho~\citep[Alg.~1]{ghoshal2018learning} %
is lower than
\XX; however \ogho\ primarily depends on the consistency of its
topological ordering rule (based on variances), which holds under
\cite[Assumption~1]{ghoshal2018learning}. In the general case where this
assumption does not hold, \eg, in challenging NV cases, the accurateness of \ogho\ 
is seriously degraded, while \XX\ (namely \XY) shows better robustness as
empirically evidenced in \Cref{ssec:exp-ev-nv}.

Another related work, \cite{varando2020learning} tackles the 
decomposition of the inverse covariance matrix through \eqref{eq:prec-b} and
maximum-likelihood of the linear SEM. \XX\ differs mainly in that it uses the support constraint in the formulation of problem~\eqref{prog:oicid+eb} 
and exploits this constraint in its optimization. %
A more remotely related approach is that of  maximum-likelihood estimation (MLE) based methods~\citep{aragam2019globally,ng2020role}. %
Note that the log-likelihood function for inverse covariance estimation (via
the Graphical Lasso formulation \citep{Friedman2007}) %
can be rewritten in terms of the linear SEM parameters $(B,\signoise)$ with
$\Theta$ as in~\eqref{eq:prec-b}. %
It follows that, ignoring the constant terms, 
$ \frac{1}{2} f(\Theta; X) := \frac{1}{2n}\trace(\trs[X]\Theta X) - \frac{1}{2}\log\det(\Theta) = -\log p(B,\signoise; X)$, 
which is the M-estimator of \golem~\citep{ng2020role} for 
optimizing $B$. 
The difference between ICID and \golem\ is that ICID proceeds with two separate steps, which in total are computationally lighter than the MLE approach of \golem. %
\XX\ has a lower per-iteration cost than \golem. Overall, it is suggested and empirically supported that \XX\ offers an improved scalability than \golem\ (\Cref{ssec:exp-scala}).

\section{Conclusion and perspectives}
This main contribution of the paper is the \XX\ method that decomposes the inverse covariance matrix $\Theta$ for learning causal structures. \XX\ leverages continuous optimization techniques for decomposing $\Theta$ subject to a support constraint~\eqref{prog:oicid+eb}. 
This approach (paired with a basic, empirical inverse covariance estimation method) demonstrated high efficiency in learning causal structures on both small and large graphs, and showed good scalability up to thousands of nodes.  
The \XX\ approach thus offers the potential to leverage future advances regarding the inverse covariance estimation step, noting that this step faces some of the key challenges (regarding, \eg, node degree distributions of the sought DAG, limited amount of observational data) for causal discovery. 

Perspectives for further research include design of regularization schemes, and ensemble approaches on the top of \XX, e.g., considering solutions obtained with different candidates for the inverse noise variance matrix $D$. The question of whether and how these solutions can be exploited will be investigated by taking inspiration from the Cause-Effect Pair Challenge \citep{CEP-Guyon}, leveraging supervised learning on the top of candidate solutions, to learn classifiers (dedicated to the considered graphs) discriminating the true edges from the spurious ones.

\appendix

\section{Proofs} 
\label{app:prfs}

\subsection{Example with chordal matrices}
\label{app:prfs-31}

The following theorem illustrates the well-definedness of the support-constrained
decomposition (\Cref{def:idecomp}) in the case where $\Theta$ is supported on a 
chordal graph. %
This result builds upon the characterization of Cholesky decomposition of
positive definite matrices supported on a chordal graph \citep{fulkerson1965incidence,rose1970triangulated,paulsen1989schur,OPT-006}. %
 
\begin{proposition} 
    \label[proposition]{prop:main-1}
Let $\Theta\in\spd_{++}$ be a positive definite matrix whose support graph $\suppo(\Theta)$ is chordal. 
Then $\Theta$ admits a support-constrained decomposition (\Cref{def:idecomp}). More over, the set
$\sset[\Theta]$~\eqref{def:set-idecomp} contains a pair $(B,D)$ where $B\in\dom$ is supported on a DAG. 
\end{proposition}

\begin{proof} 
For completeness, we state the following lemma on the connections between chordal graphs and positive definite matrices 
that can {\it factor without fill}~\citep{fulkerson1965incidence,rose1970triangulated,paulsen1989schur}. 
Given an undirected graph $\grp=(\grv,\gre)$ endowed with a node ordering $\sigma$, the following sets are considered: 
\begin{align} 
    & \mathcal{S}_{\grp,\sigma} = \{ \Theta \in \spd_{++}: \Theta_{ij} = 0 \text{~for~} (\sigma^{-1}(i),\sigma^{-1}(j)) \notin \gre\}, \label{eq:set-sgs} \\ 
    & \mathcal{L}_{\grp,\sigma} = \{ L \in \dom: L_{ii} = 1, L_{ij} = 0
    \text{~for~} i < j \text{~or~} %
     (\sigma^{-1}(i),\sigma^{-1}(j)) \notin \gre\}. \label{eq:set-lgs} 
\end{align}
\begin{lemma}[\cite{rose1970triangulated,paulsen1989schur}]
    \label{lemm:paulsen-2}  
    Let $\grp = (\grv,\gre)$ be a chordal graph, $\sigma$ an ordering of $\grv$ which corresponds to a perfect elimination ordering of $\grp$. 
    Then it holds that 
    $\Sigma \in \mathcal{S}_{\grp,\sigma}$~\eqref{eq:set-sgs} if and only if $L \in \mathcal{L}_{\grp,\sigma}$~\eqref{eq:set-lgs}, 
	where $L$ is the Cholesky factor of $\Sigma$ such that $\Sigma = LD\trs[L]$. 
\end{lemma}

%
Based on \Cref{lemm:paulsen-2}, $\grp$ being chordal implies that, %
for a certain node permutation $\sigma$ (corresponding to the permutation matrix $P_{\sigma}$), 
the positive definite matrix $\tilde{\Theta}:= P_{\sigma} \Theta \trs[P_{\sigma}]$ 
belongs to $\mathcal{S}_{\grp,\sigma_0}$~\eqref{eq:set-sgs} and that the (lower-triangular) Cholesky
factor matrix $\tilde{L}$ of $\tilde{\Theta}$ (such that $\tilde{L}\tilde{D}\trs[\tilde{L}] = \tilde{\Theta}$
for a diagonal matrix $\tilde{D}$) satisfies $\tilde{L}\in\mathcal{L}_{\grp,\sigma_0}$~\eqref{eq:set-lgs}. As a consequence,
the matrices $(A,D)$ which are $\sigma$-similar to $(\tilde{L},\tilde{D})$, \ie, $A :=
\trs[P_{\sigma}] \tilde{L} P_{\sigma}$ and $D:=\trs[P_{\sigma}] \tilde{D} P_{\sigma}$ satisfy: 
\begin{align*}
    AD\trs[A] = \trs[P_{\sigma}] \tilde{L}\tilde{D}\trs[\tilde{L}] P_{\sigma} =
\trs[P_{\sigma}] \tilde{\Theta} P_{\sigma} = \Theta,
\end{align*}
which means that $A'=A\sqrt{D}$ satisfies $A'\trs[A']=\Theta$. Moreover, it holds that $\suppo(A') \subset \suppo(\Theta)$ because (i) the two support graphs are identical to $\suppo(\tilde{L})$ and $\suppo(\tilde{\Theta})$, respectively, up to the node permutation $\sigma$ and (ii) $\suppo(\tilde{L}) \subset \suppo(\tilde{\Theta})$ by \Cref{lemm:paulsen-2}. 
Therefore $A'\trs[A']=\Theta$. 
Note that $A'$ is $\sigma$-similar to $\tilde{L}D'$ (with diagonal $D'=P_{\sigma}\sqrt{D}\trs[P_{\sigma}]$), which is a strict triangular matrix. 
Hence $A'$ represents a DAG.
\end{proof}

\subsection{Proof of \Cref{prop:thm4}}
\label{prf:prop-thm4}

\begin{proof}
    The pair $(B,\Omega)$ satisfies $\varphi(B,\Omega)=\Theta$ by definition of the SEM given. It remains to prove that $B$ also satisfies $\suppo(B)\subset \suppo(\offd{\Theta})$. Indeed, note that \Cref{assp:loh14-assp1} is equivalent to $\Theta_{ij}\neq 0$ if $(i,j)\in \mathcal{M}(B)$ (moralization of $B$). Hence, in particular, $\Theta_{i,j} \neq 0$ if $B_{i,j} \neq 0$ or $B_{j,i} \neq 0$. 
\end{proof}

\subsection{Proof of \Cref{prop:main-2}}
\label{prf:prop-main-2}
\begin{proof}
The pair satisfying $(B,{\Omega^\star}^{-1})\in\sset[{\Theta}]$ entails that 
\begin{align}
    \label{eq:theta-2bs}
    (I-B){\Omega^{\star}}^{-1} \trs[(I-B)] & = (I-B^\star) {\Omega^{\star}}^{-1} \trs[(I-B^\star)]. 
\end{align}
Now apply the following change of variable from $B$ to $\tilde{B}$: 
\begin{align*}
    (I-B) {\Omega^\star}^{-\frac{1}{2}} = {\Omega^\star}^{-\frac{1}{2}} (I-{\Omega^\star}^{\frac{1}{2}}B {\Omega^\star}^{-\frac{1}{2}}) := {\Omega^\star}^{-\frac{1}{2}} (I-\tilde{B}).
\end{align*}
The same change of variable applies to $B^\star$ such that
$(I-B^\star){\signoise^{\star}}^{-\frac{1}{2}}:= {\signoise^\star}^{-\frac{1}{2}} (I-\tilde{B}^\star)$. 
Then the equality \eqref{eq:theta-2bs} can be rewritten as 
${\signoise^\star}^{-\frac{1}{2}}(I-\tilde{B}) \trs[(I-\tilde{B})] {\signoise^\star}^{-\frac{1}{2}} 
=  {\signoise^\star}^{-\frac{1}{2}}(I-\tilde{B}_0) \trs[(I-\tilde{B}_0)]  {\signoise^\star}^{-\frac{1}{2}}$, \ie, 
\begin{align*}
    (I-\tilde{B}) \trs[(I-\tilde{B})] = (I-\tilde{B}^\star) \trs[(I-\tilde{B}^\star)]. 
\end{align*}

\revn{On the other hand, $\tilde{B}^\star$ and $\tilde{B}$ are supported on DAGs, same as $B^\star$ and $B$, respectively, since they are the results of a diagonal scaling (by the diagonal matrix ${\Omega^{\star}}^{\frac{1}{2}}$ on the left and ${\Omega^{\star}}^{-\frac{1}{2}}$ on the right) applied on $B^\star$ and $B$, which preserves the support graph.} 
Consequently, both $(I-\tilde{B}^\star)$ and $(I-\tilde{B})$ are permutation similar to unit lower-triangular matrices (lower-triangular matrices with one's on the diagonal). 
The result then follows from %
\cite[Lemma 25]{loh2014high}, which confirms that if two square matrices $X$ and $Y$ are permutation similar to unit lower-triangular matrices such that $XX^T=YY^T$, then $X=Y$. 
Through this lemma,  $I-\tilde{B}=I-\tilde{B}^\star$, which concludes the proof. 
\end{proof}

\revn{
\subsection{Proof of \Cref{coro:main-2}}
\label{prf:coro-main-2}
In the following proof, the technique using the diagonal information of $\Theta$ in the EV case can be seen in, \eg, \citep{chen2019causal}. The rest of the proof is based on \Cref{prop:main-2} on one hand, and on the other hand, the theoretical framework of \cite{ghassami2020characterizing} for proving {\it quasi-equivalence} between directed graphs from MLE-based causal discovery models. The notion of quasi-equivalence \citep[Definition~9]{ghassami2020characterizing} is a generalization of Markov equivalence (between DAGs) for directed graphs.

Given a distribution $\Theta$, depending on the parametrization with respect to the causal graph $G$, a constraint on the entries of $\Theta$ is called a {\it hard constraint} if the set of values satisfying that constraint is Lebesgue measure zero over the space of the parameters involved in the constraint. The set of hard constraints of a directed graph $G$ is denoted by $H(G)$.

The two cited assumptions, needed for deducing quasi-equivalence between directed graphs, are restated as follows. 

\noindent{\bf Assumption} (\cite[Assumptions 1--2]{ng2020role}) 
{\it Under the same notations: 
\begin{enumerate}
    \item[(i)] A distribution $\Theta$ is generalized faithful (g-faithful) to graph $G$, meaning that $\Theta$ satisfies a hard constraint $\kappa$ if and only if $\kappa \in H(G)$. 
    \item[(ii)] Let $\suppo(G)$ be the edge set of $G$.  
    For a DAG $\grp$ and a directed graph $\hat{\grp}$, it holds that 
    \begin{itemize}
        \item $|\suppo(\hat{\grp})| \leq |\suppo(\grp ) |$, then $H(\hat{\grp}) \not\subset H (\grp)$
        \item  $|\suppo(\hat{\grp})| < |\suppo(\grp ) |$, then $H(\hat{\grp}) \not\subseteq H (\grp)$
    \end{itemize}
\end{enumerate}
}
}
\revn{
\begin{proof}
    (\Cref{coro:main-2}). 
    First we prove two conditions before applying \Cref{prop:main-2}: 
    \begin{itemize}
        \item[(i)] \Cref{assp:loh14-assp1} is satisifed under the generalized faithfulness assumption (\cite[Assumption 1]{ng2020role}). The distribution $\Theta$ is such that $\Theta=\varphi(B^\star,\Omega^\star)$ where $B^\star$ is supported on a DAG, denoted as $G^*$. Let  $i\in[d]$ and $j\in[d], j\neq i$ be a pair such that $\Theta_{ij}=0$, which is a hard constraint of $\Theta$ denoted as $\kappa$. Then the g-faithfulness of $\Theta$ entails that $\kappa$ is also a hard constraint of the DAG $\grp^*$, \ie, $\kappa\in H(G^*)$. Consequently, the nodes $i$ and $j$ must not be adjacent nor spouses in $G^*$. Because otherwise, the set of values taken by the parametrization (\Cref{lemm:loh14-lemm1}) 
            $$\Theta_{ij} = -\frac{B^\star_{ij}}{{\omega_{j}^\star}^{2}} 
    - \frac{B^\star_{ji}}{{\omega_{i}^\star}^{2}} 
                    + \sum_{\ell=1}^d
                    \frac{B^\star_{i\ell}B^\star_{j\ell}}{{\omega_{\ell}^{\star}}^{2}} 
    $$
will have a nonzero Lebesgue measure, which contradicts with $\kappa$ being a hard constraint of $\Theta$. Therefore, $B^\star$ (as an adjacency matrix of $G^*$) must satisfy $B^\star_{ij} = B^\star_{ji} =0$ and $ B^\star_{i\ell} B^\star_{j\ell} = 0 $ for all $\ell$, which yields \Cref{assp:loh14-assp1}.

    \item[(ii)] It holds that $D = {\Omega^{\star}}^{-1}$, when the diagonal matrix $D$ is defined by $D_{ii} = \min_{j\in[d]} \{ \Theta_{jj}\}$, $\forall i \in [d]$. This is because when the noise variances are equal, \ie, ${\Omega^{\star}_{ii}} := \bar{\omega}^2>0$ for all $i$, then the diagonal terms of $\Theta$, from \Cref{lemm:loh14-lemm1}, are 
$\Theta_{ii} = \frac{1}{\bar{\omega}^{2}} \big(1 + \sum_{\ell=1}^d {B^\star}_{i\ell}^2\big)$. 
Hence we have $\min_{j\in[d]} \{ \Theta_{jj} \} = {\bar{\omega}}^{-2}$, which is attained because there is always a sink (a variable without child node) in the support DAG of $B^\star$. 
    \end{itemize}

Now given $D={\Omega^{\star}}^{-1}$, let $(B, D)$ be a solution to the $\ell_0$ minimization problem \eqref{prog:oicid-l0}. 
Then $(B,D)$ necessarily belongs to the constrained set $\sset[\Theta]$ \eqref{def:set-idecomp}. 
Hence it follows from \Cref{prop:main-2} that $B=B^\star$
as long as the support graph of $B$ is a DAG. 
Therefore it remains to prove that $B$ is acyclic. 
For this purpose, we evoke the proof of \cite[Theorem 1]{ng2020role} %
with necessary adaptations.  

Denote the support graph of $B$ as $G$. Since $\Theta=\varphi(B,D)$, %
$\Theta$ contains all distributional constraints of $G$. Hence $H(G) \subset H(G^*)$ under the g-faithfulness assumption. On the other hand, the solution $B$ is such that $\ellzero{B}$ is minimal among all feasible elements in $\sset[\Theta]$ including $B^\star$, hence $|\suppo(G) | \leq |\suppo(G^*)|$. By  
\cite[Assumption 2]{ng2020role} we have $H(G) \not\subset H(G^*)$. Therefore it holds that $H(G)  = H(G^*)$, and $G$ is quasi equivalent to $G^*$. 

Finally, suppose by contraction that $G$ has cycles. Let $C=(X_1,\dots,X_c,X_1)$ be a smallest (shortest) cycle on $G$. By quasi equivalence, $G^*$ and $G$ should have the same adjacencies (either via a real edge or a virtual edge~\citep{richardson1996polynomial}), $B^\star$ should also have nonzeros in the location of all the edges of~$C$.  
\begin{itemize}
    \item If $|C|>3$, then the DAG $G^*$ has a v-structure, denoted as $(\X_{i-1}\to \X_i \leftarrow \X_{i+1})$, which entails the existence of a subset of nodes $S$, $\X_i\notin S$, such that $\X_{i-1}\ind \X_{i+1} | S$. However this conditional independence relation is not true for $G$ (given the cycle). This contradicts with quasi equivalence. 

    \item If $|C|=3$, then $G$ should also have a triangle on the same triplet of nodes, which contradicts the triangle condition. 

    \item If $|C|=2$. Let $C=(X_1,X_2,X_1)$. 
If none of the adjacencies in $G$ 
to $C$ are in-going,
then $C$ can be reduced to a single edge and the resulting directed graph is equivalent to $G$~\citep{ghassami2020characterizing}. Hence, due to the edge number minimality, such $C$ is not possible. If there exists an in-going edge, say from $X_p$ to one end of $C$, there will be a virtual or real edge to the other end of $C$ as well.
Therefore, $X_p$, $X_1$, and $X_2$ are adjacent in $G$ and hence in $G^*$, which contradicts the triangle condition. Also, if the edge between $X_p$ and one end of $C$ is a virtual edge, $X_p$ should have a real edge towards another cycle in $G$, which, with the virtual edge, again forms a triangle, and hence contradicts the triangle condition.
\end{itemize}
In conclusion, in all cases, quasi equivalence or the triangle assumption is violated, which is a contradiction. Therefore the support graph of $B$ is a DAG. 
\end{proof}
}

\section{Algorithms and computational details} 
\label{sec-app:alg}

For a diagonal matrix $D\succ 0$, the smooth part of the augmented Lagrangian of
\XX~\eqref{prog:oicid+eb} or \eqref{prog:icid-reg} at $(B,\Lambda)\in\dom\times\dom$, denoted as $\augl(B,\Lambda)$, is 
\begin{align}
    \label{eq-app:lagr-icid}
\rev{\augl (B, \Lambda) = r(B) + \braket{\Lambda,\Theta-\varphi_D(B)} + \frac{\rho}{2} \fro{\Theta-\varphi_{D}(B)}^2} 
\end{align}
where $\varphi_D(B):= (I-B)D\trs[(I-B)]$ \rev{and $r(B)$ is either zero or the regularization function of \eqref{prog:icid-reg}.} 
 
The function \eqref{eq-app:lagr-icid} is linear in $\Lambda$, and the gradient
$\nabla_{\Lambda}\augl (B,\Lambda)= \Theta-\varphi_D(B)$ is used in the dual
ascent step in \Cref{alg:oicid}. 
The gradient and Hessian of $\augl(\cdot, \Lambda)$ in $B$ are given in the following proposition.

\begin{proposition}%
    \label[proposition]{prop:hess-oicid} 
For a diagonal $D\succ 0$ and $\Lambda \in \dom$, the gradient and Hessian of
$\augl(\cdot, \Lambda)$~\eqref{eq-app:lagr-icid} are as follows: 
\begin{align*}
    \nabla_B \augl(B,\Lambda) & =  2 \big(\Lambda + \rho(\Theta - \varphi_D(B)) \big) (I-B)D + \nabla_B r(B)  \\ 
    \nabla_B^2 \augl(B,\Lambda)[Z] & =  2 \Big( 4\rho\cdot\symm(ZD\trs[(I-B)]) (I-B)D -  (\Lambda+ \rho(\Theta-\varphi_D(B))) ZD  \Big)  + \nabla_B^2 r(B)[Z] 
\end{align*}
where $\symm(X) := \frac{X+\trs[X]}{2}$. 
\rev{When $r(B)$ is nonzero, for $r(B):= \lambda_2 \trace(\trs[\tilde{M}](B\odot B))$ in \eqref{prog:icid-reg}, we have 
\begin{align} 
    \label{eq:gd-h-reg}
    \nabla_B r(B) = 2\lambda_2 \tilde{M}  \odot B, \quad\text{and} \quad \nabla_B^2 r(B) [Z] = 2\lambda_2 \tilde{M} \odot Z.
\end{align}
} 
\end{proposition}

\begin{proof}
The gradient and Hessian \eqref{eq:gd-h-reg} of the regularizer $r(B)$, in the case of the regularized \XX~\eqref{prog:icid-reg}, can be obtained from straightforward calculus. 
Next, by setting the regularization back to zero, we denote the remaining part of \eqref{eq-app:lagr-icid} as  $\augl(\cdot, \Lambda)$. 
The differential of $\varphi_D(\cdot)$ in \eqref{eq-app:lagr-icid} is as follows, 
\begin{align}
    \label{eq:fdiff-oicid}
    \dop\varphi_D(B)[Z] =  - 2\cdot\symm(ZD\trs[(I-B)])
\end{align}
for any $Z\in\dom$. 
It follows that 
\begin{align}
    \label{eq:fdiff-oicid}
    \braket{\nabla_B \augl(B,\Lambda), Z} =  2\Bigl\langle{\underbrace{\Lambda + \rho(\Theta - \varphi_D(B))}_{\phi_1(B)}, \underbrace{\symm(ZD\trs[(I-B)])}_{\phi_2(B)}}\Bigl\rangle. 
\end{align}
This entails that 
\begin{align*}
\nabla_B \augl(B,\Lambda) 
 =  2 \big(\Lambda + \rho(\Theta - \varphi_D(B)) \big) (I-B)D. 
\end{align*}

For the Hessian: the differentials of $\phi_1(B)$ and $\phi_2(B)$ in
\eqref{eq:fdiff-oicid} are: 
\begin{align*}
    \dop\phi_1(B)[Z] = 2\rho \symm(ZD\trs[(I-B)])  
    \quad \text{and} \quad \dop\phi_2(B)[Z] = - ZD\trs[Z]. 
\end{align*}
Hence, by differentiating the right-hand side of \eqref{eq:fdiff-oicid}, we have 
\begin{align}
    \label{eq:zhz-oicid}
    \braket{Z, \nabla_B^2 \augl(B,\Lambda) [Z]} = 2 \Big(\underbrace{2\rho \fro{\symm(ZD\trs[(I-B)])}^2}_{a_1(B)} -  \underbrace{\braket{\Lambda+\rho (\Theta-\varphi_D(B)), ZD\trs[Z]}}_{a_2(B)} \Big). 
\end{align}
The first term $a_1(B)$ on the right-hand side is 
\begin{align*}
    a_1(B) 
    &= 2\rho \Big( \trace((I-B)D\trs[Z] \symm(ZD\trs[(I-B)]))  + \trace(ZD\trs[(I-B)] \symm(ZD\trs[(I-B)])) \Big) \\ 
    &= 2\rho \Big( \trace(\trs[Z] \symm(ZD\trs[(I-B)]) (I-B)D )  + \trace(D\trs[(I-B)] \symm(ZD\trs[(I-B)]) Z) \Big) \\ 
    &= 4\rho \trace\Big(\trs[Z] \symm(ZD\trs[(I-B)]) (I-B)D \Big), 
\end{align*}
and $a_2(B)$ is 
\begin{align*}
    a_2(B) 
    &= \trace\Big( ZD\trs[Z] (\Lambda+ \rho(\Theta-\varphi_D(B))) \Big) 
    = \trace\Big( \trs[Z] (\Lambda+ \rho(\Theta-\varphi_D(B))) ZD\Big). 
\end{align*}
Through~\eqref{eq:zhz-oicid}  we have 
\begin{align*} 
    \nabla_B^2 \augl(B,\Lambda)[Z] =  2 \Big( 4\rho\cdot\symm(ZD\trs[(I-B)]) (I-B)D -  (\Lambda+ \rho(\Theta-\varphi_D(B))) ZD  \Big).
\end{align*}
\end{proof}

\subsection{Primal problem solver for \XX} 
\label{sec-app:fista}

The FISTA \citep{beck2009fast} is used for solving~\eqref{prog:primal-oicid} in view of the $\ell_1$-norm term. The adaptation of FISTA is detailed in \Cref{alg:id-fista}.

\begin{algorithm}[htpb]
\caption{FISTA for the primal problem~\eqref{prog:primal-oicid} of \XX\label{alg:id-fista}}
    \begin{algorithmic}[1]
        \REQUIRE{Inverse covariance matrix $\Theta\in\dom$, objective function $F(\cdot):=\augl(\cdot,\Lambda)$ of~\eqref{prog:primal-oicid}, $\alpha_0>0$, $\gamma\in (0,1)$, $\beta=\frac{1}{2}$, tolerance $\epsilon$  } 
        \STATE Initialize: $\iter{B}{0} = \bm{0}_{d\times d}$, set $\iter{Y}{0}=\iter{W}{0}$. 
\label{line:fista-hof}
        \FOR{$s=1,2,\dots$ } 
        \STATE \label{algl:2-ls} Backtracking: find smallest integer $k_s\geq 0$ such that, 
for $\eta_s:= \alpha_0\beta^{k_s}$,
\hfill{\it\# see \eqref{eq:def-alpha0}} 
$$F(\tilde{B}) - F(\iter{Y}{s-1}) \leq - \gamma \eta_s \|\gd F(\iter{Y}{s-1})\|^2,
$$
where 
\begin{align*}
\tilde{B}= \prox_{\eta_s \ell_1}\big( \iter{Y}{s-1} - \eta_s \gd F(\iter{Y}{s-1})\big).
\end{align*}
\hfill{\em \# see~\eqref{eq:grad-subsp}-\eqref{eq:proxl1}} 
\STATE Update FISTA iterates: 
\begin{align*} 
    \iter{B}{s} & = \tilde{B} \quad\text{and}\quad %
    \iter{Y}{s} = \iter{B}{s} + \frac{s-1}{s+2} (\iter{B}{s} - \iter{B}{s-1}).\nonumber
\end{align*}
\STATE{Stop if $\fro{\Delta(\iter{B}{s})}\leq \epsilon$:  
\label{line:opti-b}}
\hfill{\it \# see \eqref{eq:deltab} }\\
\quad\qquad return $B_s$
\ENDFOR  { \label{line:fista-eof} }
    \end{algorithmic} 
\end{algorithm}

\paragraph{Support constraint of \XX.} 
The support constraint of \XX\ is entirely reflected in the primal problem
\eqref{prog:primal-oicid} in the form of $B \in E_{\Theta}$ where 
$\doms$~\eqref{def:subsp}, \ie,   
\begin{align} 
    \label{eq:def-subsp}
    \subc = \{ B\in\dom: B_{ij} =0 \quad \forall i=j \text{~or~} (i,j) \notin \suppo(\Theta)\}
\end{align} 
is a linear subspace of $\dom$ with dimension $(\|\Theta\|_0-d)$.  
Hence the support constraint of~\eqref{prog:primal-oicid} 
can be satisfied using subspace projection, which has the following form. 
\begin{definition}%
    \label[definition]{def:projs}
    Given (an adjacency matrix) $S\in\dom$, the projection onto the support graph $\suppo(S)$ is denoted and defined as $P_S:\dom\to\dom$ such that  
    $$(P_S(Z))_{ij} = \twopartdef{Z_{ij}}{(i,j)\in\suppo(S)}{0}{\text{otherwise.}}$$
\end{definition}

It follows that the gradient of $\augl(\cdot,\Lambda)$ restricted to subspace $\subc$, %
denoted as $\gd F(\cdot)$, is 
\begin{align}
    \label{eq:grad-subsp}
\gd F(B)=P_{\subc} (\nabla_B \augl(B,\Lambda))
\end{align}
where $\nabla_B \augl(B,\Lambda)$ is computed in~\Cref{prop:hess-oicid} 
and $P_{\subc}$ is given in \Cref{def:projs}. 

On the other hand, the proximal operator associated with the $\ell_1$ term of~\eqref{prog:primal-oicid} is 
\begin{align}
    \label{eq:proxl1}
    \prox_{\lambda \ell_1} (Z) = \twopartdef{\mathrm{sign}(Z_{ij})(|Z_{ij}| - \eta)}{|Z_{ij}|\geq \lambda}{0}{\text{otherwise.}} 
\end{align}

\paragraph{Stopping criterion.} 
\bleu{In \Cref{alg:id-fista}, line~\ref{line:opti-b}, the stopping criterion at each iteration (with a given $\Lambda\in\dom$) is defined with respect to the optimality of the primal problem~\eqref{prog:primal-oicid}. For clarity, we rewrite the smooth and non-smooth parts of the augmented Lagrangian as follows, 
$$g(x):= \ellone{x} \quad\text{and}\quad \augl(x, y) := \augl(B,\Lambda)$$
for $x:=\text{vec}(B)$ and $y:=\text{vec}(\Lambda)$. 

Then, an iterate $x$ is optimal if $- \nabla_x \augl (x, y) \in \partial g(x)$. 
This set-valued criterion can be further expressed as the following distance, for  $u := -\nabla_x \augl (x,y)$, 
\begin{align} 
    \label{eq:deltab}
    \Delta(B):= \mathrm{dist}(u, \partial g(x) ) = \left\|
    (I-P_{\partial g(x)})(u)\right\| \leq \epsilon, 
\end{align} 
where the projection term $P_{\partial g(x)}(u)$ can be obtained through the closed-form expression of $\partial g(x)$ (Cartesian product of $\partial |x_i|$ along each of the $d^2$ dimensions). Hence the expression of $\Delta(B)$ in \eqref{eq:deltab} is: for $i=1,\dots, d^2$, 
\begin{align}
((I-P_{\partial g(x)})(u))_{i}=
\begin{cases}
    1+ u_i & \text{if~} (x_i < 0) \vee ((x_i=0) \wedge (u_i < -1)) \\
    0      & \text{if~} (x_i = 0) \wedge (|u_i| \leq 1) \\  
    -1+ u_i & \text{if~} (x_i > 0) \vee ((x_i=0) \wedge (u_i > 1)). 
\end{cases}
\end{align}
}

\bleu{
\paragraph{Selection of initial stepsizes.}
\label{parag:stepsize}
We adapt the initial stepsize $\alpha_0$ in \Cref{alg:id-fista} at each iteration through the
quadratic approximation of the exact line search objective. %
More precisely, given a descent direction $Z$, we define $\alpha_0$ as the minimizer of the
following quadratic approximation of $\augl(B+\alpha Z, \Lambda)$ along $Z$, 
        \begin{align}
            \label{eq:def-alpha0}
            \alpha_0:=\min_{\alpha>0} \alpha \braket{\nabla_B \augl(B,\Lambda), Z} + \alpha^2 \braket{Z, \nabla_B^2 \augl(B,\Lambda)[Z]},
        \end{align}
        which gives $$\alpha_0 := - \frac{\braket{\nabla_B \augl(B,\Lambda), Z}} {2 \braket{Z, \nabla_B^2 \augl(B,\Lambda)[Z]}},$$
where the two terms on the right-hand side are given in
\eqref{eq:fdiff-oicid} and \eqref{eq:zhz-oicid} respectively; see proof of
\Cref{prop:hess-oicid}. 
}
In the experiments, the computation rule of $\alpha_0$~\eqref{eq:def-alpha0} is
applied during the first $10$ FISTA iterations (\Cref{alg:id-fista}, line 3)
and then the value of $\alpha_0$ is fixed at its latest value. 
The other line search parameters in \Cref{alg:id-fista} are: %
$\gamma=\frac{1}{2}$, and maximal number of backtrackings
$\bar{n}_{ls}=20$. 

\begin{remark}[Computational cost]
    \label[remark]{rmk-app:cost-alg-primal}
The computation of $\nabla_B \augl(B,\Lambda)$ (\Cref{prop:hess-oicid}) %
costs $O(k d^2)$ floating-point operations, where $k$ 
is the maximal node degree of $\Theta$, which equals the maximal Markov blanket size of $B$ under \Cref{assp:loh14-assp1}. 

\end{remark}

\bleu{
\paragraph{Parameter values.}
\label{parag-app:param-opt}
\begin{itemize}
    \item 
For \Cref{alg:id-fista}: The tolerance parameter is set as $\epsilon=10^{-4}$. The maximal number of FISTA iterations is set as 3000 for the experiments of \Cref{ssec:exp-ev-nv}, and 1000 for the experiments of \Cref{ssec:exp-scala}.

\item 
For \Cref{alg:oicid}: The tolerance parameter of the global ALM is set as
$\epsilon=5\cdot 10^{-3}$. 
The maximal number of ALM iterations is set to 3, which turns out to be empirically sufficient in the experiments of \Cref{ssec:exp-ev-nv}. 
In the experiments of \Cref{ssec:exp-scala}, the maximal
number of ALM iterations is set as 3 for \XY\ and as 1 for %
\XX\ (in \Cref{ssec:exp-scala}). 
\revn{In the experiments of \Cref{sec:fmri}, the maximal number of ALM iterations is set as 4.} 

The initial value of $\rho$ (for the augmented Lagrangian function) is $\rho := 1$. 
In \Cref{alg:oicid}, line 3, if the relative decrease in the residual $\| \Theta - \varphi_D(B)\|$ is weaker than $0.05$, then a new primal descent attempt will start after  
incrementing the value of $\rho$ by $\rho \leftarrow 1.2 \cdot \rho$. 
\end{itemize}
}

\subsection{ICID followed by DAG projection} 
\label{sec-app:icid}

\Cref{alg:icd-loram-altmin} shows a combination of ICID with a method for computing proximal DAGs. The exponential trace function $h(\cdot)$ in \notears~\cite{NEURIPS2018_e347c514} is used in the proximal DAG computation.  

\begin{algorithm}[htpb]
    \caption{DAG-constrained ICID \label{alg:icd-loram-altmin}}
    \begin{algorithmic}[1]
        \REQUIRE{Observational data $X\in\reals^{d\times n}$ } 
\STATE Get $\Theta$ from Inverse covariance estimator (\Cref{alg:ice-emp}) 
\STATE Compute $B_0$ by \XX\ (\Cref{alg:oicid}) 

\FOR{$t=1,\dots,$}
\IF{stopping criteria$(B_{t},\tilde{B}_t)$ attained}
\STATE{return $B_{t}$}
\ENDIF
\STATE Compute proximal mappings: 
\begin{align} 
    B_{t+1} & = (1-\rho)B_0 + \rho \tilde{B}_t %
   \label{eq:icd-altmin-aa} \\ 
   \tilde{B}_{t+1} & =\prox_{\gamma_2 h} ({c_0} B_{t+1}) 
   \label{eq:icd-altmin-bb} 
\end{align}
\STATE Increment $\gamma_2$ \label{algl:increm}
\ENDFOR
    \end{algorithmic}
\end{algorithm}

In this algorithm, 
the proximal mapping~\eqref{eq:icd-altmin-bb} is defined as a solution to 
\begin{align} 
    \label{eq:icd-altmin-b} 
    \min_{B\in\dom}\quad h(B) + \frac{1}{\gamma_2}\underbrace{\fro{B - c_0B_{t+1}}^2}_{g(B; B_{t+1})},
\end{align}
where $h$ is the exponential trace-based function 
\begin{align*} 
h(B) = \trace(\exp(|B|)) 
\end{align*}
with the absolute value operation $|\cdot|$ applied to $B$ element-wisely. 

Due to the exponential trace in $h$~\citep{NEURIPS2018_e347c514}, problem~\eqref{eq:icd-altmin-b} is nonconvex. We resort to the search of one proximal point~\eqref{eq:icd-altmin-bb} satisfying sufficient decrease in $h$. %
In view of reducing the cost for computing the gradients of the exponential
trace function $h$ in~\eqref{eq:icd-altmin-bb}, the low-rank method \ourmo-AGD 
of~\cite{dong2022graphs} 
is used. 
The increment rule line~\ref{algl:increm} %
is an ad-hoc adaptation of the AMA (alternating minimization algorithm) for optimizing the Lagrangian of an equality constrained optimization.  
In the experiment on large graphs, the computation of \XX\ (line 2) is limited to one single ALM iteration. 

\subsection{An empirical inverse covariance estimator} %
\label{ssec:app-sel}

In the experiments, we use a basic empirical inverse covariance estimator for construction input matrices of \XX\ and \ogho. %

\begin{algorithm}[htpb]
    \caption{Empirical inverse covariance estimator\label{alg:ice-emp}}
    \begin{algorithmic}[1]
        \REQUIRE{Data matrix $X\in\reals^{n\times d}$, parameter $\lambda_1\in (0,1)$} 
        \ENSURE{$\widehat{\Theta}_{\lambda_1}\in\dom$} 
\STATE Compute empirical covariance and its inverse: 
\begin{align}
    \label{eq:ice-emp}
    \widehat{C} = \frac{1}{n} \trs[(X-\bar{X})] (X-\bar{X})\quad \text{and} \quad
    \widehat{\Theta} = \widehat{C}^{\dagger}, %
\end{align}
where $\widehat{C}^{\dagger}$ denotes the pseudo-inverse of $\widehat{C}$. 
\STATE Element-wise thresholding on off-diagonal entries:  
\begin{align}
     \diag(\widehat{\Theta}_{\lambda_1}) & := \diag(\widehat{\Theta}), \nonumber\\
     (\widehat{\Theta}_{\lambda_1})_{\mathrm{off}} & := \mathbb{H}( \widehat{\Theta}_{\mathrm{off}}, \lambda_1 \maxn[\widehat{\Theta}_{\mathrm{off}}]),
    \label{eq:thres}
\end{align}
where $\mathbb{H}$ is defined as 
$$
\mathbb{H}(y, \tau) =  \twopartdef{y}{|y|\geq \tau}{0}{\text{otherwise.}}$$
    \end{algorithmic}
\end{algorithm}

In the computation of~\eqref{eq:ice-emp}, the pseudo-inverse coincides with the inverse of $\widehat{C}$ when $\widehat{C}$ is positive definite (\eg, when the number $n$ of samples is sufficiently large). 
In~\eqref{eq:thres}, the subscript `off' indicates the following filtering operation
$$
\Theta_{\mathrm{off}} = \{ \Theta_{ij}: i\neq j\}
$$
where the indices of the remaining (off-diagonal) entries are preserved.

\paragraph{Selection of $\lambda_1$ for \Cref{alg:ice-emp}.}
\label{ssec-app:sel-ice-emp}

In the experiment for \Cref{fig:exp1}, a parameter selection is needed. We use grid search for selecting values of $\lambda_1$ for the empirical inverse covariance estimator (\Cref{alg:ice-emp}). Note that the total time for selecting the value of $\lambda_1$ using \Cref{alg:ice-emp} is included in the computation time of \ours\ in the benchmark of \Cref{fig:exp1}. 

We start by estimating the grid search area of $\lambda_1$, based on observational data on ER graphs with $200\leq d\leq 2.10^3$ nodes. The same methodology applies to SF graphs. 

Given that most desired causal structures have an average degree $1\leq \degr\leq 4$, the target sparsity of $\widehat{\Theta}_{\lambda_1}$ by \Cref{alg:ice-emp} is bounded by $\bar{\rho}_{\degr}=\max(\frac{\degr}{d})\approx 2.0\%$ for graphs with $d\geq 200$ nodes. 
This gives us an approximate target percentile of around $98\%$, \ie, top $2\%$ edges in terms of absolute weight of $\widehat{\Theta}_{\mathrm{off}}$. 
In other words, the maximal value $\lambda_1^{\max}$ of the grid search area is set as 
$\lambda_1^{\max} :=\frac{|\widehat{\Theta}_{\mathrm{off}}(\ttau)|}{\maxn[\widehat{\Theta}_{\mathrm{off}}]}$, 
where $\ttau$ refers to the index of the $98$-th percentile in $\{|\widehat{\Theta}_{\mathrm{off}}|\}$. %
For the experiments with ER2 graphs in \Cref{sec:exps}, the estimated $\lambda_1^{\max}$ is $6.10^{-1}$. Hence, the search grid of $\lambda_1$ is set up as $n_{I_1}=20$ equidistant values on $I_1 = [10^{-2}, 6.10^{-1}]$.

\begin{figure}[H]
  \centering
  \subfigure[Distance$(\widehat{\Theta}_{\lambda_1},\Theta^{\star})$]
  {\includegraphics[width=.24\textwidth]{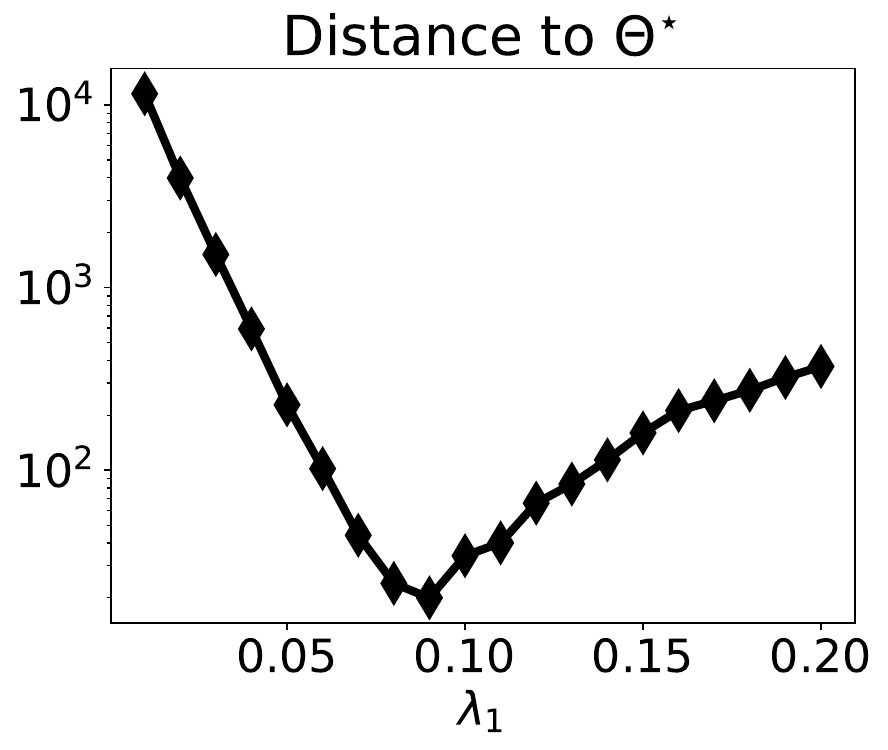}}  \qquad 
   \subfigure[Criterion used]
   {\includegraphics[width=.24\textwidth]{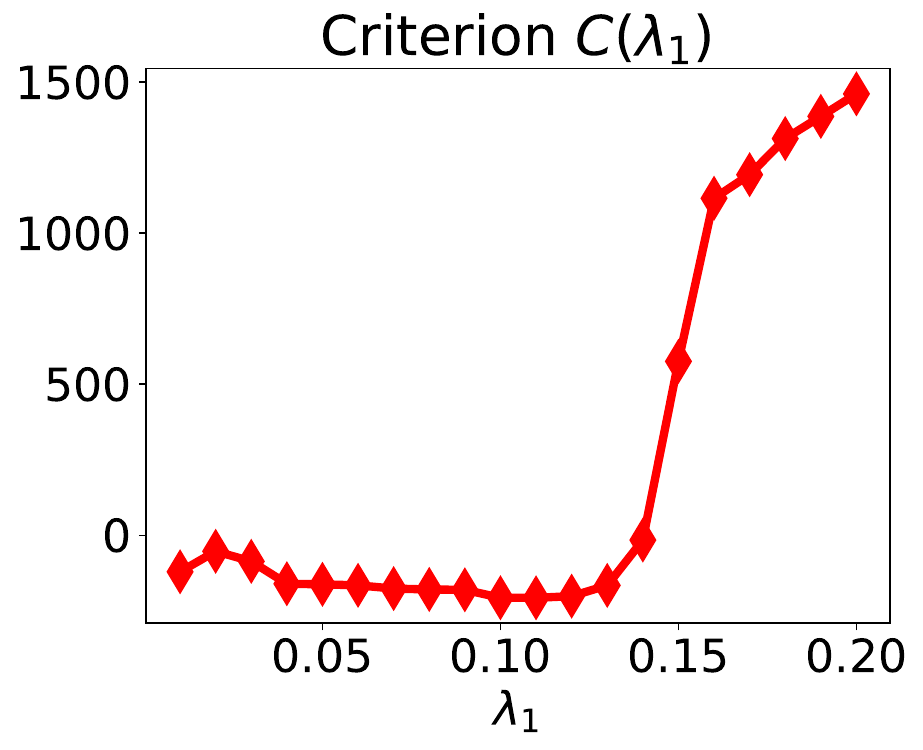}}  
    \caption{Grid search of $\lambda_1$ with \Cref{alg:ice-emp} based on criterion $C(\lambda_1)$~\eqref{eq:crit-c1}. Data $X$ is from linear \sem\ with Gaussian noise, on ER2 graph with $d=200$ nodes.}
\label{fig:app-sel-lam1}
\end{figure}

The selection criterion, similar to GraphicalLasso, is defined as 
\begin{align} \label{eq:crit-c1}
C(\lambda_1) := \trace(\widehat{C}\widehat{\Theta}_{\lambda_1}) - \log\det(\tilde{\Theta}_{\lambda_1}),
\end{align}
where $\tilde{\Theta}_{\lambda_1}  =   \widehat{\Theta}_{\lambda_1}  + \frac{9}{10} \diag(\widehat{\Theta}_{\lambda_1})$ is used in the $\log\det$-evaluation for an enhanced positive definiteness in all cases. 

\Cref{fig:app-sel-lam1} shows the criterion values compared to the Hamming distances with the oracle precision matrix $\Theta^{\star}:= \phi (B^\star)$. %
We observe that the selection criterion with $\argmin_{I_{1}} C(\lambda_1)$ gives an answer that is rather close to the optimal value in terms of distance of $\widehat{\Theta}_{\lambda_1}$ to the oracle precision matrix $\Theta^{\star}$.

\subsection{The skewness-based regularization}
\label{ssec-app:icid-reg}

The matrix $\tilde{M}$ in the regularized \XX~\eqref{prog:icid-reg}, based on third-order cumulant statistics \citep[Eq. (18)]{hyvarinen2013pairwise}, is the negative of the following matrix: 
    \begin{align}
        \label{eq:mat-cumul}
            M = C \odot \expe[\X\trs[g(\X)] -  g(\X)\trs[\X]]
    \end{align}
where $C$ denotes the correlation matrix of $\X$ and $g(\cdot)$ applies to $\X$ elementwisely. 

\paragraph{Effect of skewness on the matrix $M$.}
Suppose that the pairwise model between $x:=X_i$ and $y:=X_j$ is 
            $y = \rho x + u$ if $x\to y$ and is $ x = \rho y + u$ if $y\to x$, where $|\rho| <1$. 
The third-order cumulant-based statistic \citep[Eq. (10)]{hyvarinen2013pairwise} is defined as follows: 
\begin{align}
    \label{eq:def-rc3}
\tilde{R}_{c3} (x,y) :=  \rho \expe [ x^2 y - x y^2 ].
\end{align}
Then the $(i,j)$-th term of $M$ is 
$$ M_{ij} = C_{ij} \expe[X_i^2X_j - X_iX_j^2]  = \frac{1}{\rho} C_{ij}\tilde{R}_{c3}(x,y) $$
which is positive if $X_i \to X_j$ and otherwise is negative. 

On the other hand, the following theorem shows the connection between $\tilde{R}_{c3}$ and the causal parameter $\rho$ depending on the skewness: 
                \begin{theorem}[{\cite[Theorem 2]{hyvarinen2013pairwise}}]
                    \label{thm:h13}
                    If the causal direction is $x\to y$, then
                    $$ \tilde{R}_{c3} = \mathrm{skew}(x) (\rho^2 - \rho^3)$$
                    and if the causal direction is the opposite, then 
                    $$ \tilde{R}_{c3} = \mathrm{skew}(y) (\rho^3 - \rho^2).$$
                \end{theorem}
Therefore, when the variables of $\X$ are positively skewed, we obtain the following table about the signs of $M_{ij}$ and signs of $\rho$. 
    \begin{table}[H]
        \centering
        \label{tab:mij-reg} 
        \caption{Signs of $M_{ij}$ for different signs of $\rho$ and different pairwise causal directions. }
        \begin{tabular}{cc|ccr}
        \hline
          Direction    & $\rho (|\rho|<1)$   & $C_{ij}$ & $\tilde{R}_{3c}$ & $M_{ij}$   \\
        \hline                                               
         $x\to y$  &   $>0$              &   +   &    +   &  +  \\ 
         $x\to y$  &   $<0$              &   -   &    +   &  +  \\ 
         $x\leftarrow y$  &   $>0$       &   +   &    -   &  -  \\ 
         $x\leftarrow y$  &   $<0$       &   -   &    -   &  -  \\
        \hline
        \end{tabular}
    \end{table}
As the table above suggests, the sign of $M_{ij}$ is mostly positive if the causal direction is $(X_i \to X_j)$ and otherwise it is negative.  

This motivates us to define the regularization term of the regularized \XX~\eqref{prog:icid-reg} as 
$$ r(B) := \trace( \trs[\tilde{M}] (B\odot B)  ).$$
Since with positive skewness, the trace value $r(B)$ is small (possibly negative) when the pairwise causal directions in $B$ are consistent with the signs of the pairwise measures encoded in $M= -\tilde{M}$, and is large when the pairwise directions in $B$ are reversed. 

\subsection{Parameter selection for \XX\ and \XSkew} 

We optimize \XSkew, the regularized \XX~\eqref{prog:icid-reg}, using the same algorithm (\Cref{alg:oicid}) as \XX. Computationally, \XSkew\ differs with \XX\ only in the additional regularization function $r(B)$, which is subsumed in the proposed FISTA subroutine (\Cref{alg:id-fista}) according to \Cref{prop:hess-oicid}. 

Both \XX\ and \XSkew\ uses as input the estimated inverse covariance matrix by \Cref{alg:ice-emp}, which has a selection rule for the parameter $\lambda_1$ given a set values; see \Cref{ssec-app:sel-ice-emp}. 
In the experiment on sythetic data, the set for selection is composed of the 10 linearly equal distant values $\lambda_1 \in [0.04, 0.1]$. 
In the experiment on the fMRI datasets, the set for selection is 
10 linearly equal distant values of $\lambda_1 \in [0.01, 0.1]$.

Subsequently, both \XX\ and \XSkew\ are optimized by \Cref{alg:oicid} for a parameter $\rho>0$, which determines the initial weight of the $\ell_1$ term in the augmented Lagrangian. In practice, we used a small set of values $\lambda'_1:=\frac{1}{\rho} \in \{0.05, 0.1,0.2,0.3\}$ for the experiments on synthetic data (\Cref{ssec:exp-scala}), and  
$\lambda'_1:=\frac{1}{\rho} \in \{0.01, 0.05, 0.06, 0.07, 0.08, 0.1, 0.2\}$ for the experiments on the fMRI datasets. 

The one additional regularization parameter $\lambda_2$ for \XSkew\ is also selected by a simple grid search, given a fixed pair of preselected parameters $(\lambda_1, \rho)$ of \XX. 
The grid search result for $\lambda_2$ on the fMRI Sim4 data is as follows: 

\begin{table}[htpb]
\centering
\caption{Grid search for the regularization parameter $\lambda_2$ of \XSkew.
$\lambda_1$ is the parameter for sparse inverse covariance estimation (\Cref{alg:ice-emp}) and $\lambda'_1:= 1/\rho$ is the ALM parameter in \Cref{alg:oicid}. 
    \label{tab:exp-fmri-icidreg}}
\begin{tabular}{c|c|ccccc}
\hline
{\it fMRI Sim4}    & ($\lambda_1$, $\lambda'_1$, $\lambda_2$) & SHD & TPR & FDR & nnz  &  Primal optimality \\ 
\hline                                               
\XX\  & $(0.0256, 0.05, 0.00 )$          &  14  & 0.820 &  0.219  &  64  &  1.1e-5   \\   
\XX\ (skew) & $(0.0256, 0.05, 0.01 )$    &  18  & 0.787 &  0.273  &  66  &  2.1e-5   \\
\XX\ (skew) & $(0.0256, 0.05, 0.02 )$    &  12  & 0.934 &  0.173  &  69  &  6.2e-5   \\
\XX\ (skew) & $(0.0256, 0.05, 0.03 )$    &   7  & 0.967 &  0.106  &  66  &  2.3e-5   \\
\XX\ (skew) & $(0.0256, 0.05, 0.05 )$    &   5  & 0.984 &  0.076  &  65  &  9.2e-5   \\
\XX\ (skew) & $(0.0256, 0.05, 0.09 )$    &   6  & 0.984 &  0.090  &  66  &  5.0e-6   \\
\hline
\end{tabular}
\end{table}

The learning accuracy of the estimated graphs during the grid search are visualized in the following ROC curves.
        \begin{figure}[H]
            \centering
            \includegraphics[width=.76\textwidth]{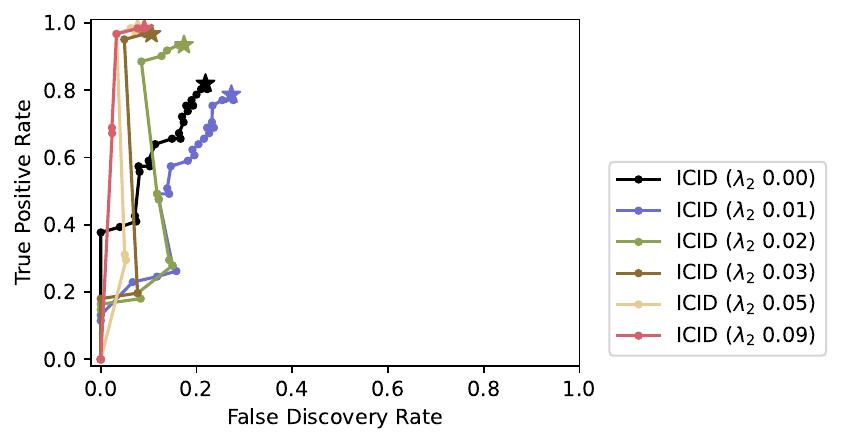} 
            \caption{ROC curves of \XSkew\ with different values of $\lambda_2$ on the fMRI Sim4 dataset.\label{fig:roc-icidreg}. The x-axis is the FDR scores instead of FPRs.} 
        \end{figure}

\section{Experiments}

The codes of the proposed algorithms are available at 
https://anonymous.4open.science/r/icid-v2-9611.

\subsection{Random graphs and synthetic data}
\label{ssec:data-graphs}
The experiments on synthetic data are conducted with DAGs generated from two sets of random graphs: (i) Erd\H{o}s--R\'enyi (ER) graphs and (ii) Scale-free (SF)~\citep{barabasi1999emergence} graphs, as characterized in \Cref{tab:graph-types}. %
\begin{table}[!htbp]
\small
\centering 
\caption{Features of different graph models.}
\label{tab:graph-types}
\begin{tabular}{c|cc}
\hline
                    & Parameter     & Degree distribution   \\ \hline
                    Erd\H{o}s-R\'enyi   & $p\in (0,1)$   & Binomial $\mathcal{B}(d,p)$          \\  
Scale-free          & $\gamma$       & $P(k)\propto k^{-\gamma}$ \\
\hline
\end{tabular}
\end{table}

The generation of random DAGs from the two sets above is the same as in~\citep{NEURIPS2018_e347c514}. The naming of these graphs has a node degree specification, such as `ER1', where the number indicates the average node degree of the graph. 
Specifically, for a given DAG $\grp^\star$, 
its weighted adjacency matrix $B^\star$ is generated by assigning weights to the nonzeros of $\mathbb{B}(\grp^\star)\in\{0,1\}^{d\times d}$ independently from the uniform distribution: 
$B^\star_{ij}\sim \text{Unif}([-2,-0.5]\cup [0.5, 2]), \text{~for~} (i,j)\in \suppo(\mathbb{B}(\grp^\star))$. 
We generate observational data according to the linear \sem\ model~\eqref{eq:sem}, and store them in dataset $X\in\reals^{n\times d}$ where $n$ is the number of samples. The additive noises of the linear \sem\, such that
$X = \trs[{B^\star}] X + E$, %
belong to either of the following models: (i) Gaussian noise ($\gau$): $E \sim\mathcal{N}(0, \signoise)$ %
and (ii) Exponential noise ($\expp$): $E\sim \text{Exp}(\signoise)$, %
where $\signoise$ is a diagonal matrix of noise variances. 
Therefore, the dataset $X$ belongs to one of the following categories $\{\er\degr, \sff\degr\}\times \{\gau, \expp\}$ (where $\degr$ is the aforementioned average node degree).

\subsection{Evaluation Metrics} 
\label{ssec-app:metrics}

The graph metrics for the comparison of graph edge sets are the commonly used (\eg, by the aforementioned baseline methods) ones as follows: \\ 
(1) TPR = TP/T                      \quad (higher is better), \\ 
(2)  FDR = (R + FP)/P           \quad  (lower is better),  \\ 
(3) FPR = (R + FP)/F            \quad  (lower is better),  \\ 
(4) SHD = E +M + R              \quad (lower is better). \\  
More precisely, SHD is the (minimal) total number of edge additions (E), deletions (M), and reversals (R) needed to convert an estimated DAG into a true DAG. 
Since a pair of directed graphs are compared, a distinction between True Positives (TP) and Reversed edges (R) is needed: the former is estimated with correct direction
whereas the latter is not. Likewise, a False Positive (FP) is an edge that is not in the undirected skeleton of the true graph. In addition, Positive (P) is the set of estimated edges, True (T) is the set of true edges, False (F) is the set of non-edges in the ground truth graph. Finally, let (E) be the extra
edges from the skeleton, (M) be the missing edges from the skeleton.

\subsection{Causal structure learning from inverse covariance matrices} %
\label{ssec-app:exp-ev-nv}

The results of the experiment in \Cref{ssec:exp-ev-nv} are also represented in \Cref{tab:std-interp} and \Cref{tab:std-interp-er3}. 

\begin{table}[htpb]
\small 
\centering
\caption{Results of causal structure learning from inverse covariance matrices in the EV case ($\lambda=0$) and the NV cases ($\lambda>0$). The scaling parameter $\lambda$ is defined in \eqref{eq:def-xlam}. \label{tab:std-interp}}
\vspace{1mm}
\begin{tabular}{l|c|c||cccc}
\hline
G  & $\lambda$                                                                              & Algorithm    &  TPR             &  SHD                &  Median SHD   &  time (seconds)               \\ 
\hline
\multicolumn{1}{c|}{\multirow{15}{*}{ER1}} 	& \multicolumn{1}{c|}{\multirow{3}{*}{0.00}}    & \ogho\   & 1.000 {\scriptsize$\pm$ 0.000}    	& 0.000 {\scriptsize$\pm$ 0.000}    	& 0.000        	& 0.006 {\scriptsize$\pm$ 0.003}  \\ \cline{3-7}         	
\multicolumn{1}{c|}{                     } 	& \multicolumn{1}{c|}{                     } 	& GES   & 0.894 {\scriptsize$\pm$ 0.053}    	& 31.700 {\scriptsize$\pm$ 7.103}    	& 29.000        	& 386.155 {\scriptsize$\pm$ 56.533}  \\ \cline{3-7}
\multicolumn{1}{c|}{                     } 	& \multicolumn{1}{c|}{                     } 	& \XY\    & 0.922 {\scriptsize$\pm$ 0.057}    	& 9.800 {\scriptsize$\pm$ 5.350}    	& 10.500        	& 1.616 {\scriptsize$\pm$ 0.771}  \\ \cline{2-7}
\multicolumn{1}{c|}{                     } 	& \multicolumn{1}{c|}{\multirow{3}{*}{0.10}}  	& \ogho\   & 0.976 {\scriptsize$\pm$ 0.026}    	& 3.500 {\scriptsize$\pm$ 4.378}    	& 2.000        	& 0.006 {\scriptsize$\pm$ 0.003}  \\ \cline{3-7}
\multicolumn{1}{c|}{                     } 	& \multicolumn{1}{c|}{                     }  	& GES   & 0.894 {\scriptsize$\pm$ 0.053}    	& 31.700 {\scriptsize$\pm$ 7.103}    	& 29.000        	& 269.355 {\scriptsize$\pm$ 80.731}  \\ \cline{3-7}
\multicolumn{1}{c|}{                     } 	& \multicolumn{1}{c|}{                     }  	& \XY\    & 0.958 {\scriptsize$\pm$ 0.047}    	& 4.900 {\scriptsize$\pm$ 5.446}    	& 4.000        	& 1.336 {\scriptsize$\pm$ 0.431}  \\ \cline{2-7}
\multicolumn{1}{c|}{                     } 	& \multicolumn{1}{c|}{\multirow{3}{*}{0.20}}  	& \ogho\   & 0.896 {\scriptsize$\pm$ 0.071}    	& 13.900 {\scriptsize$\pm$ 9.433}    	& 11.500        	& 0.005 {\scriptsize$\pm$ 0.006}  \\ \cline{3-7}
\multicolumn{1}{c|}{                     } 	& \multicolumn{1}{c|}{                     }  	& GES   & 0.894 {\scriptsize$\pm$ 0.053}    	& 31.700 {\scriptsize$\pm$ 7.103}    	& 29.000        	& 223.234 {\scriptsize$\pm$ 65.922}  \\ \cline{3-7}
\multicolumn{1}{c|}{                     } 	& \multicolumn{1}{c|}{                     }  	& \XY\    & 0.970 {\scriptsize$\pm$ 0.057}    	& 2.900 {\scriptsize$\pm$ 5.953}    	& 0.000        	& 1.321 {\scriptsize$\pm$ 0.345}  \\ \cline{2-7}
\multicolumn{1}{c|}{                     } 	& \multicolumn{1}{c|}{\multirow{3}{*}{0.40}}  	& \ogho\   & 0.736 {\scriptsize$\pm$ 0.082}    	& 36.000 {\scriptsize$\pm$ 13.174}    	& 33.000        	& 0.004 {\scriptsize$\pm$ 0.001}  \\ \cline{3-7}
\multicolumn{1}{c|}{                     } 	& \multicolumn{1}{c|}{                     }  	& GES   & 0.894 {\scriptsize$\pm$ 0.053}    	& 31.700 {\scriptsize$\pm$ 7.103}    	& 29.000        	& 202.746 {\scriptsize$\pm$ 44.534}  \\ \cline{3-7}
\multicolumn{1}{c|}{                     } 	& \multicolumn{1}{c|}{                     }  	& \XY\    & 0.996 {\scriptsize$\pm$ 0.013}    	& 0.500 {\scriptsize$\pm$ 1.581}    	& 0.000        	& 1.741 {\scriptsize$\pm$ 0.243}  \\ \cline{2-7}
\multicolumn{1}{c|}{                     } 	& \multicolumn{1}{c|}{\multirow{3}{*}{0.80}}  	& \ogho\   & 0.514 {\scriptsize$\pm$ 0.071}    	& 65.500 {\scriptsize$\pm$ 12.457}    	& 62.500        	& 0.004 {\scriptsize$\pm$ 0.001}  \\ \cline{3-7}
\multicolumn{1}{c|}{                     } 	& \multicolumn{1}{c|}{                     }  	& GES   & 0.894 {\scriptsize$\pm$ 0.053}    	& 31.700 {\scriptsize$\pm$ 7.103}    	& 29.000        	& 183.254 {\scriptsize$\pm$ 28.068}  \\ \cline{3-7}
\multicolumn{1}{c|}{                     } 	& \multicolumn{1}{c|}{                     }  	& \XY\    & 1.000 {\scriptsize$\pm$ 0.000}    	& 0.000 {\scriptsize$\pm$ 0.000}    	& 0.000        	& 1.452 {\scriptsize$\pm$ 0.424}  \\ \cline{2-7}
\multicolumn{1}{c|}{                     } 	& \multicolumn{1}{c|}{\multirow{3}{*}{1.00}}  	& \ogho\   & 0.414 {\scriptsize$\pm$ 0.087}    	& 82.600 {\scriptsize$\pm$ 17.424}    	& 81.500        	& 0.005 {\scriptsize$\pm$ 0.003}  \\ \cline{3-7}
\multicolumn{1}{c|}{                     } 	& \multicolumn{1}{c|}{                     }  	& GES   & 0.894 {\scriptsize$\pm$ 0.053}    	& 31.700 {\scriptsize$\pm$ 7.103}    	& 29.000        	& 181.916 {\scriptsize$\pm$ 28.672}  \\ \cline{3-7}
\multicolumn{1}{c|}{                     } 	& \multicolumn{1}{c|}{                     }  	& \XY\    & 0.992 {\scriptsize$\pm$ 0.017}    	& 0.400 {\scriptsize$\pm$ 0.843}    	& 0.000        	& 1.251 {\scriptsize$\pm$ 0.282}  \\ \cline{2-7}
\hline
\multicolumn{1}{c|}{\multirow{15}{*}{ER2}} 	& \multicolumn{1}{c|}{\multirow{3}{*}{0.00}}     & \ogho\   & 1.000 {\scriptsize$\pm$ 0.000}    	& 0.000 {\scriptsize$\pm$ 0.000}    	& 0.000        	& 0.004 {\scriptsize$\pm$ 0.000}  \\ \cline{3-7}         
\multicolumn{1}{c|}{                     } 	& \multicolumn{1}{c|}{                     } 	 & GES   & 0.828 {\scriptsize$\pm$ 0.082}    	& 73.889 {\scriptsize$\pm$ 31.700}    	& 69.000        	& 1076.755 {\scriptsize$\pm$ 1063.638}  \\ \cline{3-7}
\multicolumn{1}{c|}{                     } 	& \multicolumn{1}{c|}{                     } 	 & \XY\    & 0.827 {\scriptsize$\pm$ 0.113}    	& 43.700 {\scriptsize$\pm$ 23.443}    	& 44.000        	& 1.651 {\scriptsize$\pm$ 0.860}  \\ \cline{2-7}
\multicolumn{1}{c|}{                     } 	& \multicolumn{1}{c|}{\multirow{3}{*}{0.10}}  	 & \ogho\   & 0.912 {\scriptsize$\pm$ 0.044}    	& 49.700 {\scriptsize$\pm$ 36.999}    	& 32.500        	& 0.004 {\scriptsize$\pm$ 0.000}  \\ \cline{3-7}
\multicolumn{1}{c|}{                     } 	& \multicolumn{1}{c|}{                     }  	 & GES   & 0.828 {\scriptsize$\pm$ 0.082}    	& 73.889 {\scriptsize$\pm$ 31.700}    	& 69.000        	& 1074.745 {\scriptsize$\pm$ 1060.711}  \\ \cline{3-7}
\multicolumn{1}{c|}{                     } 	& \multicolumn{1}{c|}{                     }  	 & \XY\    & 0.920 {\scriptsize$\pm$ 0.048}    	& 28.500 {\scriptsize$\pm$ 19.196}    	& 31.000        	& 1.605 {\scriptsize$\pm$ 0.888}  \\ \cline{2-7}
\multicolumn{1}{c|}{                     } 	& \multicolumn{1}{c|}{\multirow{3}{*}{0.20}}  	 & \ogho\   & 0.794 {\scriptsize$\pm$ 0.058}    	& 107.300 {\scriptsize$\pm$ 24.622}    	& 109.500        	& 0.005 {\scriptsize$\pm$ 0.003}  \\ \cline{3-7}
\multicolumn{1}{c|}{                     } 	& \multicolumn{1}{c|}{                     }  	 & GES   & 0.828 {\scriptsize$\pm$ 0.082}    	& 73.889 {\scriptsize$\pm$ 31.700}    	& 69.000        	& 1074.792 {\scriptsize$\pm$ 1059.921}  \\ \cline{3-7}
\multicolumn{1}{c|}{                     } 	& \multicolumn{1}{c|}{                     }  	 & \XY\    & 0.911 {\scriptsize$\pm$ 0.101}    	& 23.500 {\scriptsize$\pm$ 23.904}    	& 20.000        	& 1.945 {\scriptsize$\pm$ 0.941}  \\ \cline{2-7}
\multicolumn{1}{c|}{                     } 	& \multicolumn{1}{c|}{\multirow{3}{*}{0.40}}  	 & \ogho\   & 0.585 {\scriptsize$\pm$ 0.074}    	& 202.000 {\scriptsize$\pm$ 40.877}    	& 215.000        	& 0.004 {\scriptsize$\pm$ 0.001}  \\ \cline{3-7}
\multicolumn{1}{c|}{                     } 	& \multicolumn{1}{c|}{                     }  	 & GES   & 0.828 {\scriptsize$\pm$ 0.082}    	& 73.889 {\scriptsize$\pm$ 31.700}    	& 69.000        	& 1074.323 {\scriptsize$\pm$ 1060.670}  \\ \cline{3-7}
\multicolumn{1}{c|}{                     } 	& \multicolumn{1}{c|}{                     }  	 & \XY\    & 0.964 {\scriptsize$\pm$ 0.082}    	& 12.400 {\scriptsize$\pm$ 23.081}    	& 2.000        	& 2.512 {\scriptsize$\pm$ 1.227}  \\ \cline{2-7}
\multicolumn{1}{c|}{                     } 	& \multicolumn{1}{c|}{\multirow{3}{*}{0.80}}  	 & \ogho\   & 0.320 {\scriptsize$\pm$ 0.036}    	& 362.500 {\scriptsize$\pm$ 36.035}    	& 377.500        	& 0.004 {\scriptsize$\pm$ 0.002}  \\ \cline{3-7}
\multicolumn{1}{c|}{                     } 	& \multicolumn{1}{c|}{                     }  	 & GES   & 0.828 {\scriptsize$\pm$ 0.082}    	& 73.889 {\scriptsize$\pm$ 31.700}    	& 69.000        	& 1074.660 {\scriptsize$\pm$ 1061.174}  \\ \cline{3-7}
\multicolumn{1}{c|}{                     } 	& \multicolumn{1}{c|}{                     }  	 & \XY\    & 0.940 {\scriptsize$\pm$ 0.048}    	& 18.800 {\scriptsize$\pm$ 17.364}    	& 10.000        	& 3.092 {\scriptsize$\pm$ 0.586}  \\ \cline{2-7}
\multicolumn{1}{c|}{                     } 	& \multicolumn{1}{c|}{\multirow{3}{*}{1.00}}  	 & \ogho\   & 0.263 {\scriptsize$\pm$ 0.035}    	& 412.400 {\scriptsize$\pm$ 32.925}    	& 412.000        	& 0.004 {\scriptsize$\pm$ 0.000}  \\ \cline{3-7}
\multicolumn{1}{c|}{                     } 	& \multicolumn{1}{c|}{                     }  	 & GES   & 0.828 {\scriptsize$\pm$ 0.082}    	& 73.889 {\scriptsize$\pm$ 31.700}    	& 69.000        	& 1073.877 {\scriptsize$\pm$ 1060.595}  \\ \cline{3-7}
\multicolumn{1}{c|}{                     } 	& \multicolumn{1}{c|}{                     }  	 & \XY\    & 0.884 {\scriptsize$\pm$ 0.072}    	& 25.900 {\scriptsize$\pm$ 15.409}    	& 26.500        	& 3.248 {\scriptsize$\pm$ 0.824}  \\ \cline{2-7}
\hline
\end{tabular}
\end{table}

\begin{table}[H]
\small 
\centering
\caption{Results of causal structure learning from inverse covariance matrices in the EV case ($\lambda=0$) and the NV cases ($\lambda>0$). The scaling parameter $\lambda$ is defined in \eqref{eq:def-xlam}. \label{tab:std-interp-er3}} 
\vspace{1mm}
\begin{tabular}{l|c|c||cccc}
\hline
G  & $\lambda$                                                                              & Algorithm    &  TPR             &  SHD                &  Median SHD   &  time (seconds)               \\ 
\hline
\multicolumn{1}{c|}{\multirow{15}{*}{ER3}} 	& \multicolumn{1}{c|}{\multirow{3}{*}{0.00}} 	& \ogho\   & 1.000 {\scriptsize$\pm$ 0.000}    	& 0.000 {\scriptsize$\pm$ 0.000}    	& 0.000        	& 0.004 {\scriptsize$\pm$ 0.000}  \\ \cline{3-7}          
\multicolumn{1}{c|}{                     } 	& \multicolumn{1}{c|}{                     } 	& GES   & 0.663 {\scriptsize$\pm$ 0.042}    	& 215.000 {\scriptsize$\pm$ 53.740}    	& 215.000        	& 72466.820 {\scriptsize$\pm$ 10675.060}  \\ \cline{3-7}
\multicolumn{1}{c|}{                     } 	& \multicolumn{1}{c|}{                     }    & \XY\    & 0.733 {\scriptsize$\pm$ 0.114}    	& 109.600 {\scriptsize$\pm$ 44.890}    	& 106.500        	& 2.526 {\scriptsize$\pm$ 0.582}  \\ \cline{2-7}
\multicolumn{1}{c|}{                     } 	& \multicolumn{1}{c|}{\multirow{3}{*}{0.10}}  	& \ogho\   & 0.945 {\scriptsize$\pm$ 0.030}    	& 67.800 {\scriptsize$\pm$ 51.352}    	& 47.000        	& 0.006 {\scriptsize$\pm$ 0.004}  \\ \cline{3-7}
\multicolumn{1}{c|}{                     } 	& \multicolumn{1}{c|}{                     }  	& fastGES   & 0.802 {\scriptsize$\pm$ 0.064}    	& 314.200 {\scriptsize$\pm$ 41.590}    	& 315.500        	& 2.527 {\scriptsize$\pm$ 0.432}  \\ \cline{3-7}
\multicolumn{1}{c|}{                     } 	& \multicolumn{1}{c|}{                     }    & \XY\    & 0.867 {\scriptsize$\pm$ 0.067}    	& 65.000 {\scriptsize$\pm$ 33.113}    	& 57.500        	& 2.133 {\scriptsize$\pm$ 0.791}  \\ \cline{2-7}  	
\multicolumn{1}{c|}{                     } 	& \multicolumn{1}{c|}{\multirow{3}{*}{0.20}}  	& \ogho\   & 0.768 {\scriptsize$\pm$ 0.094}    	& 236.900 {\scriptsize$\pm$ 80.733}    	& 242.500        	& 0.005 {\scriptsize$\pm$ 0.002}  \\ \cline{3-7}
\multicolumn{1}{c|}{                     } 	& \multicolumn{1}{c|}{                     }  	& fastGES   & 0.802 {\scriptsize$\pm$ 0.064}    	& 314.200 {\scriptsize$\pm$ 41.590}    	& 315.500        	& 2.378 {\scriptsize$\pm$ 0.475}  \\ \cline{3-7}
\multicolumn{1}{c|}{                     } 	& \multicolumn{1}{c|}{                     }    & \XY\    & 0.849 {\scriptsize$\pm$ 0.117}    	& 76.600 {\scriptsize$\pm$ 58.043}    	& 62.000        	& 2.761 {\scriptsize$\pm$ 1.151}  \\ \cline{2-7}  	
\multicolumn{1}{c|}{                     } 	& \multicolumn{1}{c|}{\multirow{3}{*}{0.40}}  	& \ogho\   & 0.501 {\scriptsize$\pm$ 0.068}    	& 424.200 {\scriptsize$\pm$ 66.406}    	& 411.500        	& 0.005 {\scriptsize$\pm$ 0.003}  \\ \cline{3-7}
\multicolumn{1}{c|}{                     } 	& \multicolumn{1}{c|}{                     }  	& fastGES   & 0.802 {\scriptsize$\pm$ 0.064}    	& 314.200 {\scriptsize$\pm$ 41.590}    	& 315.500        	& 2.432 {\scriptsize$\pm$ 0.637}  \\ \cline{3-7}
\multicolumn{1}{c|}{                     } 	& \multicolumn{1}{c|}{                     }    & \XY\    & 0.910 {\scriptsize$\pm$ 0.078}    	& 66.400 {\scriptsize$\pm$ 55.191}    	& 53.000        	& 3.415 {\scriptsize$\pm$ 0.509}  \\ \cline{2-7}  	
\multicolumn{1}{c|}{                     } 	& \multicolumn{1}{c|}{\multirow{3}{*}{0.80}}  	& \ogho\   & 0.229 {\scriptsize$\pm$ 0.034}    	& 641.100 {\scriptsize$\pm$ 65.467}    	& 635.500        	& 0.005 {\scriptsize$\pm$ 0.003}  \\ \cline{3-7}
\multicolumn{1}{c|}{                     } 	& \multicolumn{1}{c|}{                     }  	& fastGES   & 0.802 {\scriptsize$\pm$ 0.064}    	& 314.200 {\scriptsize$\pm$ 41.590}    	& 315.500        	& 2.462 {\scriptsize$\pm$ 0.531}  \\ \cline{3-7}
\multicolumn{1}{c|}{                     } 	& \multicolumn{1}{c|}{                     }    & \XY\    & 0.862 {\scriptsize$\pm$ 0.082}    	& 93.400 {\scriptsize$\pm$ 29.444}    	& 94.500        	& 3.950 {\scriptsize$\pm$ 0.645}  \\ \cline{2-7}  	
\multicolumn{1}{c|}{                     } 	& \multicolumn{1}{c|}{\multirow{3}{*}{1.00}}  	& \ogho\   & 0.176 {\scriptsize$\pm$ 0.040}    	& 703.400 {\scriptsize$\pm$ 54.271}    	& 703.000        	& 0.005 {\scriptsize$\pm$ 0.003}  \\ \cline{3-7}
\multicolumn{1}{c|}{                     } 	& \multicolumn{1}{c|}{                     }  	& fastGES   & 0.802 {\scriptsize$\pm$ 0.064}    	& 314.200 {\scriptsize$\pm$ 41.590}    	& 315.500        	& 2.382 {\scriptsize$\pm$ 0.439}  \\ \cline{3-7}
\multicolumn{1}{c|}{                     } 	& \multicolumn{1}{c|}{                     }    & \XY\    & 0.705 {\scriptsize$\pm$ 0.153}    	& 88.100 {\scriptsize$\pm$ 22.293}    	& 86.500        	& 3.627 {\scriptsize$\pm$ 0.943}  \\ \cline{2-7}  	
\hline
\end{tabular}
\end{table}

In addition to the results on ER1 and ER3 sets (\Cref{fig:oicid-std-interp-flexD-v-other2-upd}), the following figure (\Cref{fig:oicid-std-interp-flexD-v-other2-upd-er2}) shows the results on the ER2 set. 
\begin{figure}[H]
    \centering
    \subfigure[ER2]{\includegraphics[width=0.68\textwidth]{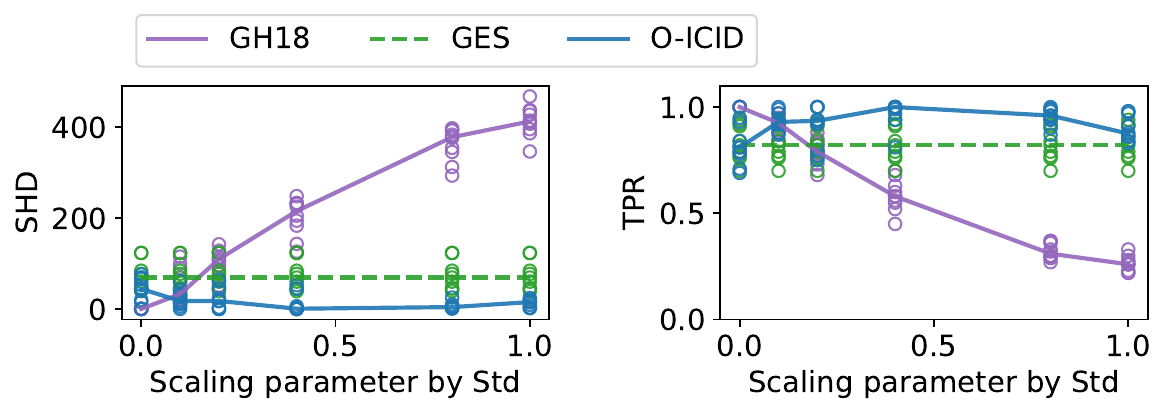}}
    \caption{Causal structure learning by \XY\ compared to \ogho\ and GES. The
    X-axis indicates the scaling parameter $\lambda\in[0,1]$ in
    \eqref{eq:def-xlam}. The true DAGs are drawn from the
    ER2 set with $d=50$ nodes.}
    \label{fig:oicid-std-interp-flexD-v-other2-upd-er2} 
\end{figure}

\paragraph{Same experiment on SF graphs.} 

The experiment of \Cref{ssec:exp-ev-nv} is also conducted with random DAGs drawn from the SF sets. The results are presented in 
\Cref{fig:oicid-std-interp-flexD-v-other2-upd-sf} and \Cref{tab:std-interp-sf}.  In particular, GES is tested only with $\lambda=0$ (due to running time limits), and its scores for the other values of $\lambda$ remain the same as $\lambda=0$ in view of the invariance of GES to the value of $\lambda$ (note that $\lambda$ only controls the scaling of each $X_i$'s samples). Therefore the running time of GES is only reported for $\lambda=0$ in \Cref{tab:std-interp-sf}. 

\begin{figure}[H]
    \centering
    \subfigure[SF2]{\includegraphics[width=0.74\textwidth]{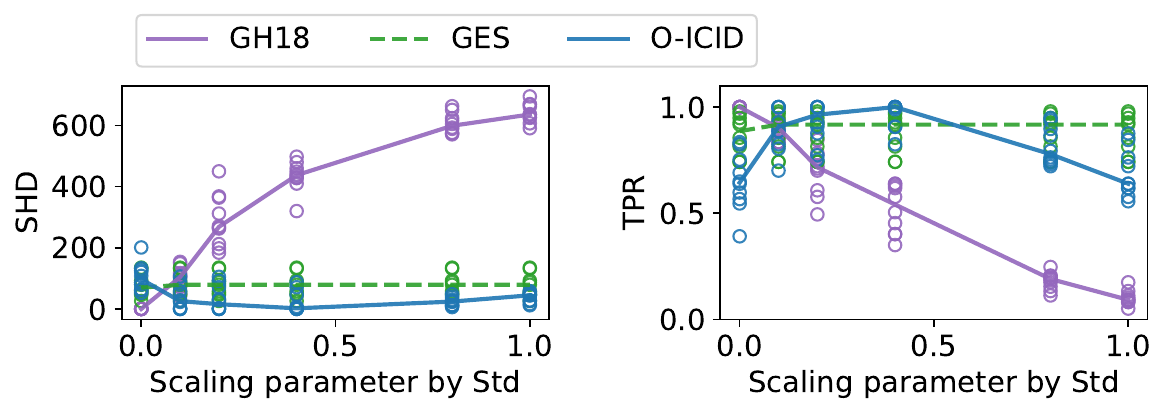}}
    \subfigure[SF4]{\includegraphics[width=0.74\textwidth]{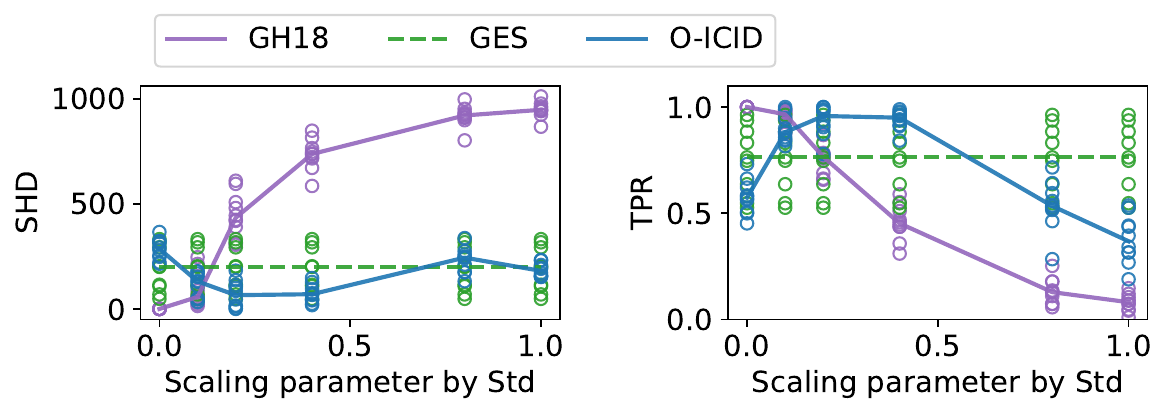}}
    \caption{Causal structure learning by \XY\ compared to \ogho\ and GES. The
    X-axis indicates $\lambda\in[0,1]$ in
    \eqref{eq:def-xlam}. The true DAGs are drawn from the
    SF2 and SF4 sets with $d=50$ nodes.}
    \label{fig:oicid-std-interp-flexD-v-other2-upd-sf} 
\end{figure}

\begin{table}[H]
\footnotesize 
\centering
\caption{Results of causal structure learning from inverse covariance matrices in the EV case ($\lambda=0$) and the NV cases ($\lambda>0$). The scaling parameter $\lambda$ is defined in \eqref{eq:def-xlam}. The true DAGs are drawn from the SF sets. In particular, GES is tested only with $\lambda=0$. \label{tab:std-interp-sf}} 
\vspace{1mm}
\begin{tabular}{l|c|c||cccc}
\hline
G  & $\lambda$                                                                              & Algorithm    &  TPR             &  SHD                &  Median SHD   &  time (seconds)               \\ 
\hline
\multicolumn{1}{c|}{\multirow{15}{*}{SF2}} 	& \multicolumn{1}{c|}{\multirow{3}{*}{0.00}}     & \ogho\   & 1.000 {\scriptsize$\pm$ 0.000}    	& 0.000 {\scriptsize$\pm$ 0.000}    	& 0.000        	& 0.005 {\scriptsize$\pm$ 0.001}  \\ \cline{3-7}          
\multicolumn{1}{c|}{                     } 	& \multicolumn{1}{c|}{                     } 	 & GES   & 0.877 {\scriptsize$\pm$ 0.093}    	& 78.556 {\scriptsize$\pm$ 37.627}    	& 79.000        	& 742.761 {\scriptsize$\pm$ 281.321}  \\ \cline{3-7}
\multicolumn{1}{c|}{                     } 	& \multicolumn{1}{c|}{                     } 	 & \XY\    & 0.649 {\scriptsize$\pm$ 0.135}    	& 103.300 {\scriptsize$\pm$ 44.697}    	& 97.000        	& 10.445 {\scriptsize$\pm$ 1.617}  \\ \cline{2-7}
\multicolumn{1}{c|}{                     } 	& \multicolumn{1}{c|}{\multirow{3}{*}{0.10}}  	 & \ogho\   & 0.897 {\scriptsize$\pm$ 0.045}    	& 97.600 {\scriptsize$\pm$ 38.390}    	& 104.000        	& 0.004 {\scriptsize$\pm$ 0.000}  \\ \cline{3-7}
\multicolumn{1}{c|}{                     } 	& \multicolumn{1}{c|}{                     }  	 & GES   & 0.877 {\scriptsize$\pm$ 0.093}    	& 78.556 {\scriptsize$\pm$ 37.627}    	& 79.000        	& NA  \\ \cline{3-7}
\multicolumn{1}{c|}{                     } 	& \multicolumn{1}{c|}{                     }  	 & \XY\    & 0.903 {\scriptsize$\pm$ 0.103}    	& 38.400 {\scriptsize$\pm$ 41.097}    	& 25.500        	& 8.523 {\scriptsize$\pm$ 3.172}  \\ \cline{2-7}
\multicolumn{1}{c|}{                     } 	& \multicolumn{1}{c|}{\multirow{3}{*}{0.20}}  	 & \ogho\   & 0.696 {\scriptsize$\pm$ 0.106}    	& 288.800 {\scriptsize$\pm$ 85.036}    	& 268.500        	& 0.003 {\scriptsize$\pm$ 0.000}  \\ \cline{3-7}
\multicolumn{1}{c|}{                     } 	& \multicolumn{1}{c|}{                     }  	 & GES   & 0.877 {\scriptsize$\pm$ 0.093}    	& 78.556 {\scriptsize$\pm$ 37.627}    	& 79.000        	& NA  \\ \cline{3-7}
\multicolumn{1}{c|}{                     } 	& \multicolumn{1}{c|}{                     }  	 & \XY\    & 0.924 {\scriptsize$\pm$ 0.096}    	& 33.200 {\scriptsize$\pm$ 41.335}    	& 15.500        	& 6.136 {\scriptsize$\pm$ 3.583}  \\ \cline{2-7}
\multicolumn{1}{c|}{                     } 	& \multicolumn{1}{c|}{\multirow{3}{*}{0.40}}  	 & \ogho\   & 0.522 {\scriptsize$\pm$ 0.113}    	& 434.900 {\scriptsize$\pm$ 48.045}    	& 437.500        	& 0.003 {\scriptsize$\pm$ 0.000}  \\ \cline{3-7}
\multicolumn{1}{c|}{                     } 	& \multicolumn{1}{c|}{                     }  	 & GES   & 0.877 {\scriptsize$\pm$ 0.093}    	& 78.556 {\scriptsize$\pm$ 37.627}    	& 79.000        	& NA  \\ \cline{3-7}
\multicolumn{1}{c|}{                     } 	& \multicolumn{1}{c|}{                     }  	 & \XY\    & 0.963 {\scriptsize$\pm$ 0.062}    	& 21.600 {\scriptsize$\pm$ 33.533}    	& 2.000        	& 6.801 {\scriptsize$\pm$ 3.615}  \\ \cline{2-7}
\multicolumn{1}{c|}{                     } 	& \multicolumn{1}{c|}{\multirow{3}{*}{0.80}}  	 & \ogho\   & 0.181 {\scriptsize$\pm$ 0.039}    	& 606.200 {\scriptsize$\pm$ 32.155}    	& 599.000        	& 0.003 {\scriptsize$\pm$ 0.000}  \\ \cline{3-7}
\multicolumn{1}{c|}{                     } 	& \multicolumn{1}{c|}{                     }  	 & GES   & 0.877 {\scriptsize$\pm$ 0.093}    	& 78.556 {\scriptsize$\pm$ 37.627}    	& 79.000        	& NA  \\ \cline{3-7}
\multicolumn{1}{c|}{                     } 	& \multicolumn{1}{c|}{                     }  	 & \XY\    & 0.815 {\scriptsize$\pm$ 0.088}    	& 23.300 {\scriptsize$\pm$ 13.483}    	& 24.000        	& 8.691 {\scriptsize$\pm$ 2.169}  \\ \cline{2-7}
\multicolumn{1}{c|}{                     } 	& \multicolumn{1}{c|}{\multirow{3}{*}{1.00}}  	 & \ogho\   & 0.102 {\scriptsize$\pm$ 0.033}    	& 641.600 {\scriptsize$\pm$ 31.697}    	& 636.000        	& 0.003 {\scriptsize$\pm$ 0.000}  \\ \cline{3-7}
\multicolumn{1}{c|}{                     } 	& \multicolumn{1}{c|}{                     }  	 & GES   & 0.877 {\scriptsize$\pm$ 0.093}    	& 78.556 {\scriptsize$\pm$ 37.627}    	& 79.000        	& NA  \\ \cline{3-7}
\multicolumn{1}{c|}{                     } 	& \multicolumn{1}{c|}{                     }  	 & \XY\    & 0.686 {\scriptsize$\pm$ 0.111}    	& 37.800 {\scriptsize$\pm$ 16.075}    	& 44.500        	& 8.517 {\scriptsize$\pm$ 2.748}  \\ \cline{2-7}
\hline                                                                                                                                                                                                                                                                 
\multicolumn{1}{c|}{\multirow{15}{*}{SF4}} 	& \multicolumn{1}{c|}{\multirow{3}{*}{0.00}}     & \ogho\   & 1.000 {\scriptsize$\pm$ 0.000}    	& 0.000 {\scriptsize$\pm$ 0.000}    	& 0.000        	& 0.003 {\scriptsize$\pm$ 0.000}  \\ \cline{3-7}
\multicolumn{1}{c|}{                     } 	& \multicolumn{1}{c|}{                     } 	 & GES   & 0.760 {\scriptsize$\pm$ 0.161}    	& 187.222 {\scriptsize$\pm$ 108.545}    	& 200.000        	& 2407.557 {\scriptsize$\pm$ 1876.750}  \\ \cline{3-7}
\multicolumn{1}{c|}{                     } 	& \multicolumn{1}{c|}{                     } 	 & \XY\    & 0.577 {\scriptsize$\pm$ 0.082}    	& 283.300 {\scriptsize$\pm$ 48.847}    	& 286.000        	& 9.755 {\scriptsize$\pm$ 2.319}  \\ \cline{2-7}
\multicolumn{1}{c|}{                     } 	& \multicolumn{1}{c|}{\multirow{3}{*}{0.10}}  	 & \ogho\   & 0.951 {\scriptsize$\pm$ 0.041}    	& 93.100 {\scriptsize$\pm$ 82.332}    	& 57.000        	& 0.003 {\scriptsize$\pm$ 0.000}  \\ \cline{3-7}
\multicolumn{1}{c|}{                     } 	& \multicolumn{1}{c|}{                     }  	 & GES   & 0.760 {\scriptsize$\pm$ 0.161}    	& 187.222 {\scriptsize$\pm$ 108.545}    	& 200.000        	& NA  \\ \cline{3-7}
\multicolumn{1}{c|}{                     } 	& \multicolumn{1}{c|}{                     }  	 & \XY\    & 0.896 {\scriptsize$\pm$ 0.063}    	& 117.700 {\scriptsize$\pm$ 56.210}    	& 131.000        	& 10.094 {\scriptsize$\pm$ 1.770}  \\ \cline{2-7}
\multicolumn{1}{c|}{                     } 	& \multicolumn{1}{c|}{\multirow{3}{*}{0.20}}  	 & \ogho\   & 0.751 {\scriptsize$\pm$ 0.064}    	& 448.700 {\scriptsize$\pm$ 103.395}    	& 434.000        	& 0.003 {\scriptsize$\pm$ 0.000}  \\ \cline{3-7}
\multicolumn{1}{c|}{                     } 	& \multicolumn{1}{c|}{                     }  	 & GES   & 0.760 {\scriptsize$\pm$ 0.161}    	& 187.222 {\scriptsize$\pm$ 108.545}    	& 200.000        	&  NA                                     \\ \cline{3-7}
\multicolumn{1}{c|}{                     } 	& \multicolumn{1}{c|}{                     }  	 & \XY\    & 0.936 {\scriptsize$\pm$ 0.071}    	& 69.400 {\scriptsize$\pm$ 63.533}    	& 65.500        	& 9.070 {\scriptsize$\pm$ 3.105}  \\ \cline{2-7}
\multicolumn{1}{c|}{                     } 	& \multicolumn{1}{c|}{\multirow{3}{*}{0.40}}  	 & \ogho\   & 0.457 {\scriptsize$\pm$ 0.083}    	& 734.500 {\scriptsize$\pm$ 72.529}    	& 737.000        	& 0.003 {\scriptsize$\pm$ 0.000}  \\ \cline{3-7}
\multicolumn{1}{c|}{                     } 	& \multicolumn{1}{c|}{                     }  	 & GES   & 0.760 {\scriptsize$\pm$ 0.161}    	& 187.222 {\scriptsize$\pm$ 108.545}    	& 200.000        	&  NA                                    \\ \cline{3-7}
\multicolumn{1}{c|}{                     } 	& \multicolumn{1}{c|}{                     }  	 & \XY\    & 0.941 {\scriptsize$\pm$ 0.043}    	& 71.100 {\scriptsize$\pm$ 42.344}    	& 70.000        	& 9.676 {\scriptsize$\pm$ 2.179}  \\ \cline{2-7}
\multicolumn{1}{c|}{                     } 	& \multicolumn{1}{c|}{\multirow{3}{*}{0.80}}  	 & \ogho\   & 0.132 {\scriptsize$\pm$ 0.059}    	& 917.300 {\scriptsize$\pm$ 49.218}    	& 920.500        	& 0.003 {\scriptsize$\pm$ 0.001}  \\ \cline{3-7}
\multicolumn{1}{c|}{                     } 	& \multicolumn{1}{c|}{                     }  	 & GES   & 0.760 {\scriptsize$\pm$ 0.161}    	& 187.222 {\scriptsize$\pm$ 108.545}    	& 200.000        	& NA                                      \\ \cline{3-7}
\multicolumn{1}{c|}{                     } 	& \multicolumn{1}{c|}{                     }  	 & \XY\    & 0.536 {\scriptsize$\pm$ 0.114}    	& 230.000 {\scriptsize$\pm$ 58.752}    	& 244.000        	& 10.670 {\scriptsize$\pm$ 3.277}  \\ \cline{2-7}
\multicolumn{1}{c|}{                     } 	& \multicolumn{1}{c|}{\multirow{3}{*}{1.00}}  	 & \ogho\   & 0.080 {\scriptsize$\pm$ 0.040}    	& 947.900 {\scriptsize$\pm$ 37.063}    	& 947.500        	& 0.003 {\scriptsize$\pm$ 0.000}  \\ \cline{3-7}
\multicolumn{1}{c|}{                     } 	& \multicolumn{1}{c|}{                     }  	 & GES   & 0.760 {\scriptsize$\pm$ 0.161}    	& 187.222 {\scriptsize$\pm$ 108.545}    	& 200.000        	&  NA                                    \\ \cline{3-7}
\multicolumn{1}{c|}{                     } 	& \multicolumn{1}{c|}{                     }  	 & \XY\    & 0.377 {\scriptsize$\pm$ 0.111}    	& 185.300 {\scriptsize$\pm$ 31.245}    	& 181.000        	& 9.124 {\scriptsize$\pm$ 1.707}  \\ \cline{2-7}
\hline
\end{tabular}
\end{table}

\paragraph{Misspecification of noise variances.} 

The following figures show the sensitivity of \XY\ under misspecification of noise variances in all tested configurations: the random DAGs are drawn from ER$k$ sets with $k\in\{1,2,3\}$, and the scaling parameter $\lambda$ via standardization ranges in $\{0.1, 0.2, 0.4, 0.8, 1.0\}$. 

\begin{figure}[H]
    \centering
    \subfigure[ER1, $\lambda=0.1$]{\includegraphics[width=0.490\textwidth]{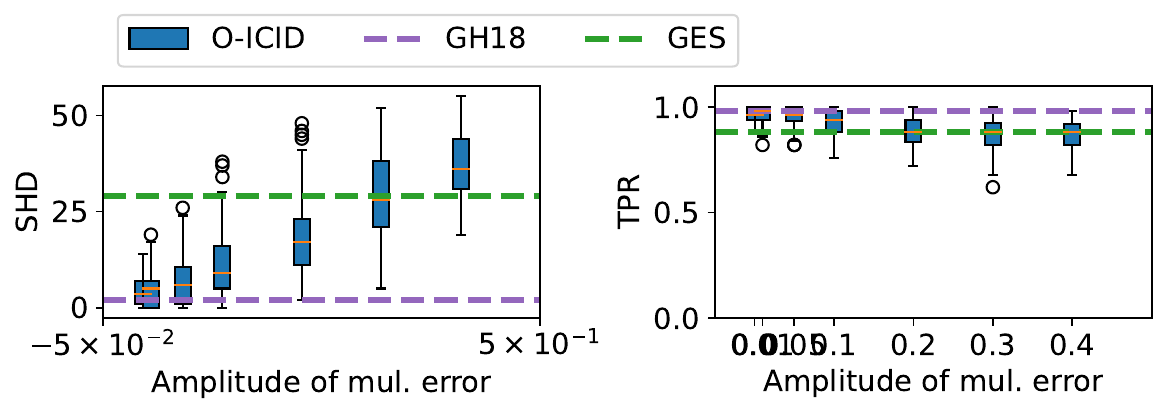}} 
    \subfigure[ER1, $\lambda=0.2$]{\includegraphics[width=0.490\textwidth]{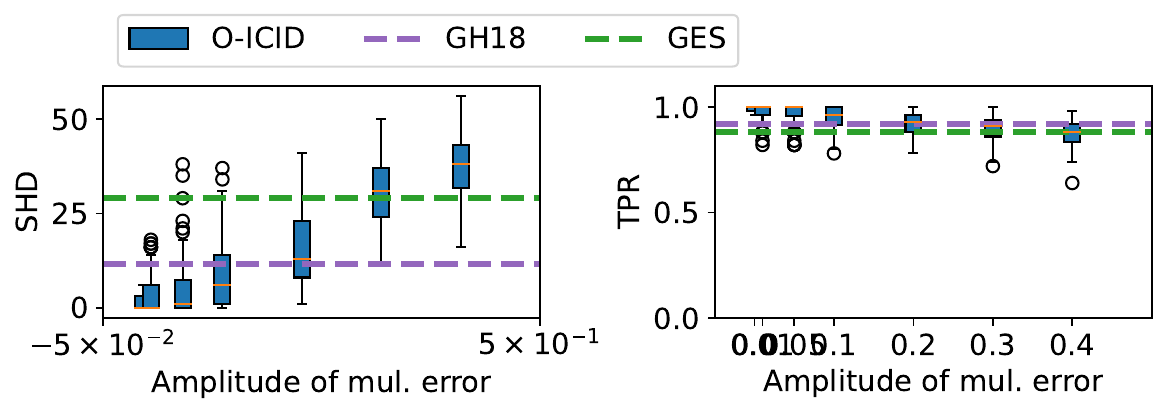}} 
    \\
    \subfigure[ER1, $\lambda=0.4$]{\includegraphics[width=0.490\textwidth]{figures/exp/oicid_v3exp7_sens_dstar/scat7_iset4.pdf}} 
    \subfigure[ER1, $\lambda=0.8$]{\includegraphics[width=0.490\textwidth]{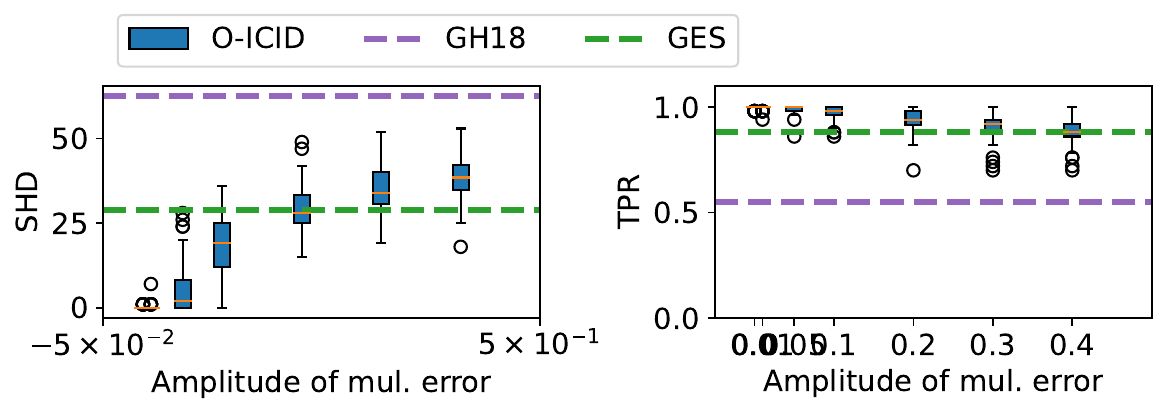}} 
    \\ 
    \subfigure[ER2, $\lambda=0.1$]{\includegraphics[width=0.490\textwidth]{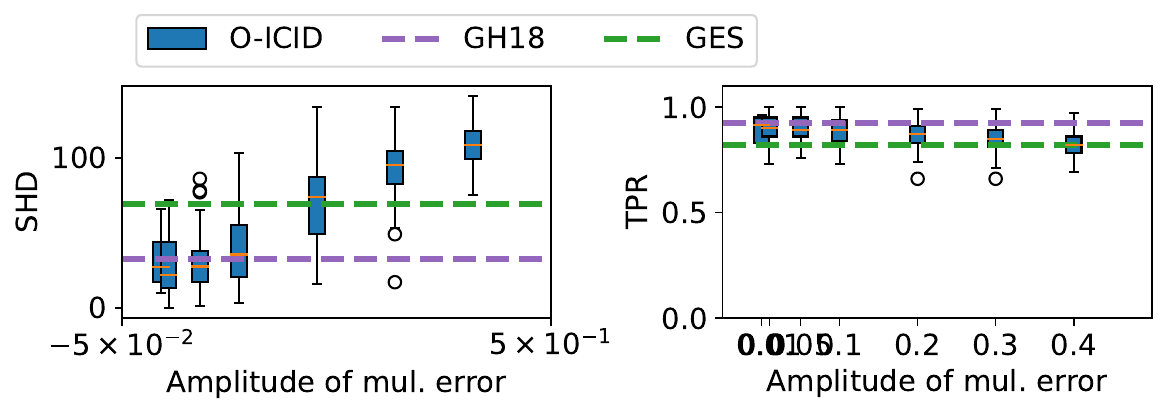}} 
    \subfigure[ER2, $\lambda=0.2$]{\includegraphics[width=0.490\textwidth]{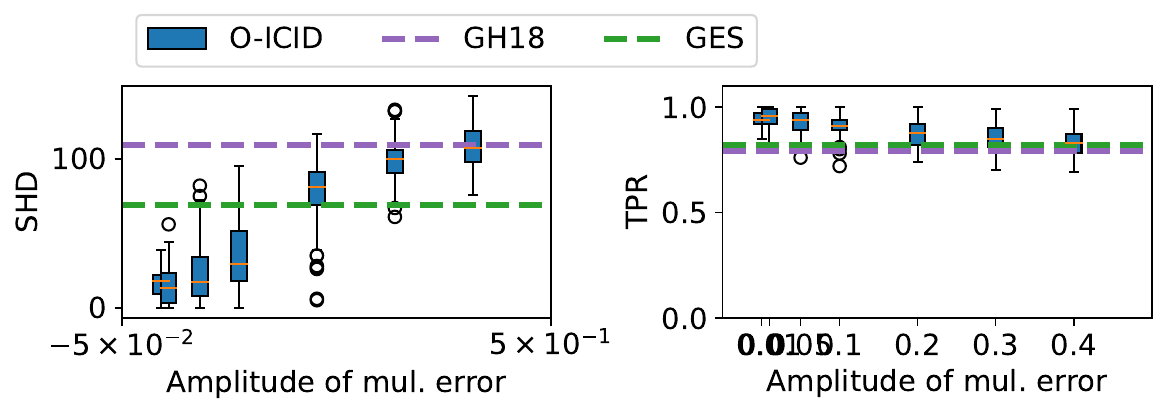}} 
    \\
    \subfigure[ER2, $\lambda=0.4$]{\includegraphics[width=0.490\textwidth]{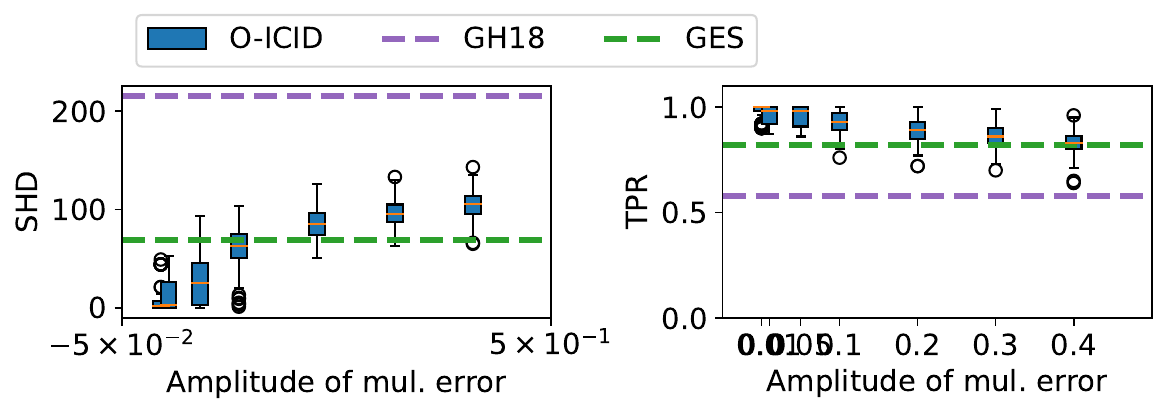}} 
    \subfigure[ER2, $\lambda=0.8$]{\includegraphics[width=0.490\textwidth]{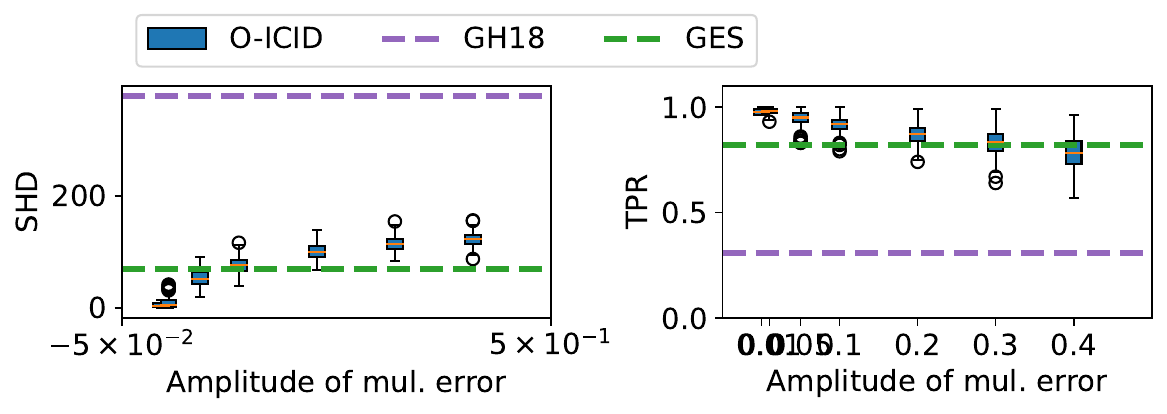}} 
    \caption{Causal structure learning by \XY. The X-axis indicates the
    amplitude $\sigma_{\epsilon}$ of average misspecifications of~$D^*$.\label{fig:oicid-std-interp-Dpert-2} }
\end{figure}

\begin{figure}[H]
    \centering
    \subfigure[ER3, $\lambda=0.1$]{\includegraphics[width=0.490\textwidth]{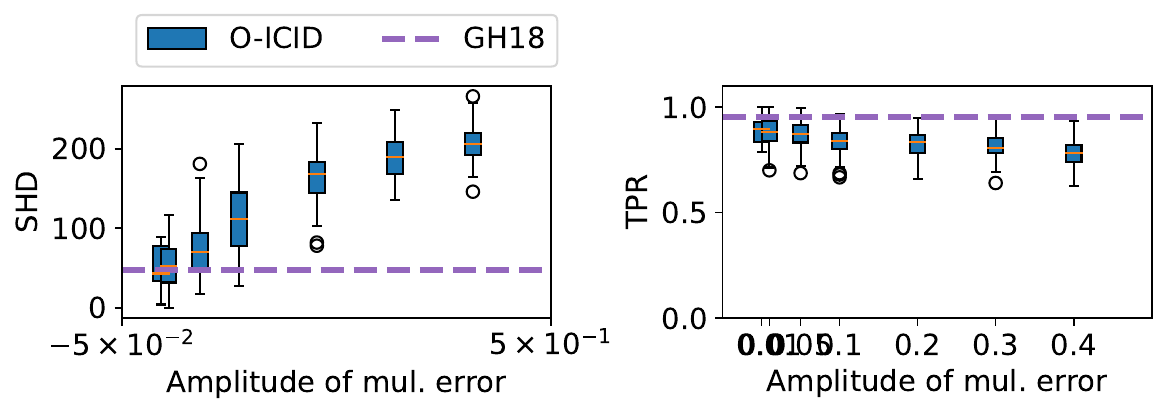}} 
    \subfigure[ER3, $\lambda=0.2$]{\includegraphics[width=0.490\textwidth]{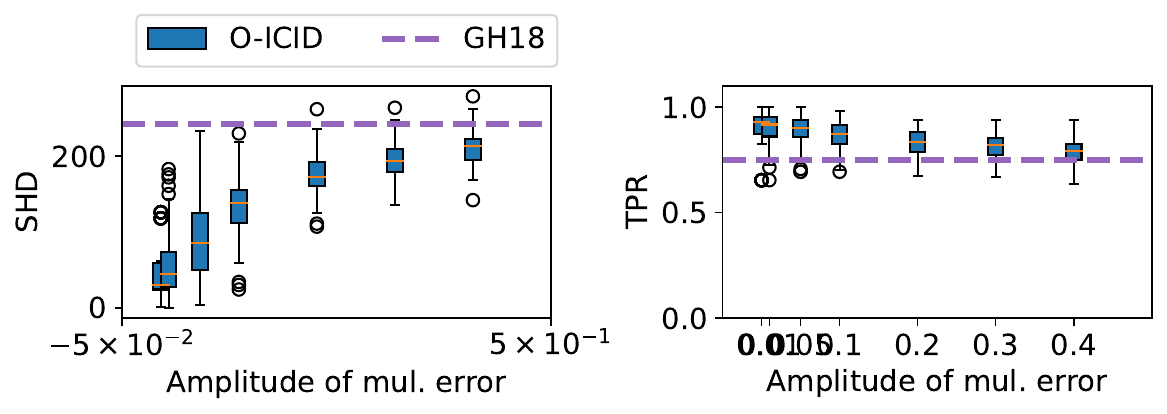}} 
    \\
    \subfigure[ER3, $\lambda=0.4$]{\includegraphics[width=0.490\textwidth]{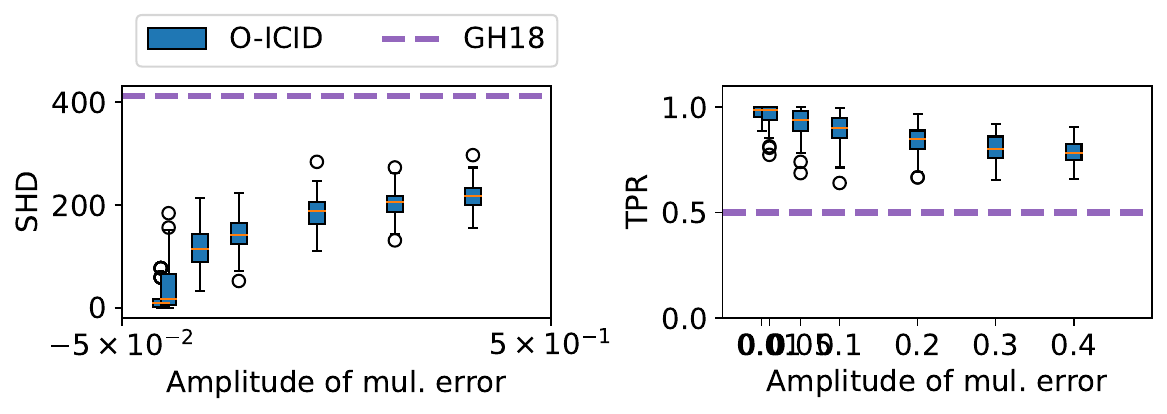}} 
    \subfigure[ER3, $\lambda=0.8$]{\includegraphics[width=0.490\textwidth]{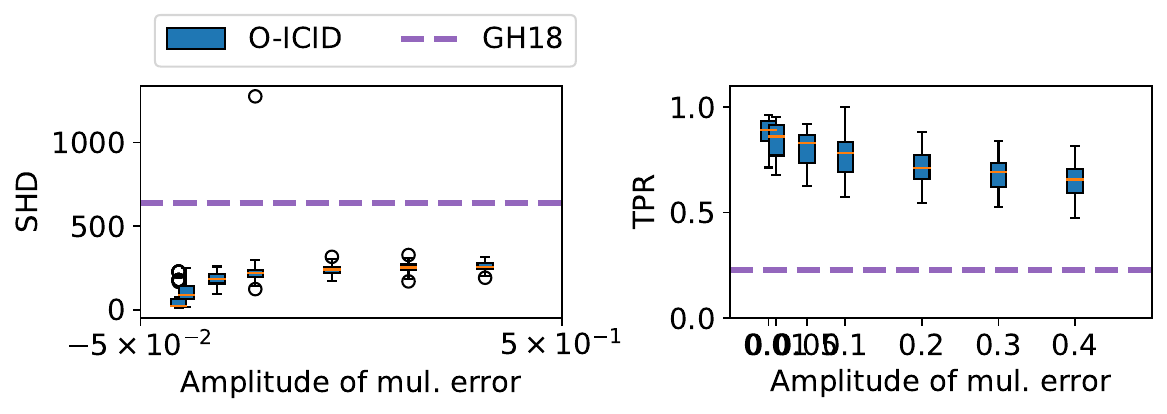}} 
    \\
    \subfigure[ER3, $\lambda=1$]{\includegraphics[width=0.490\textwidth]{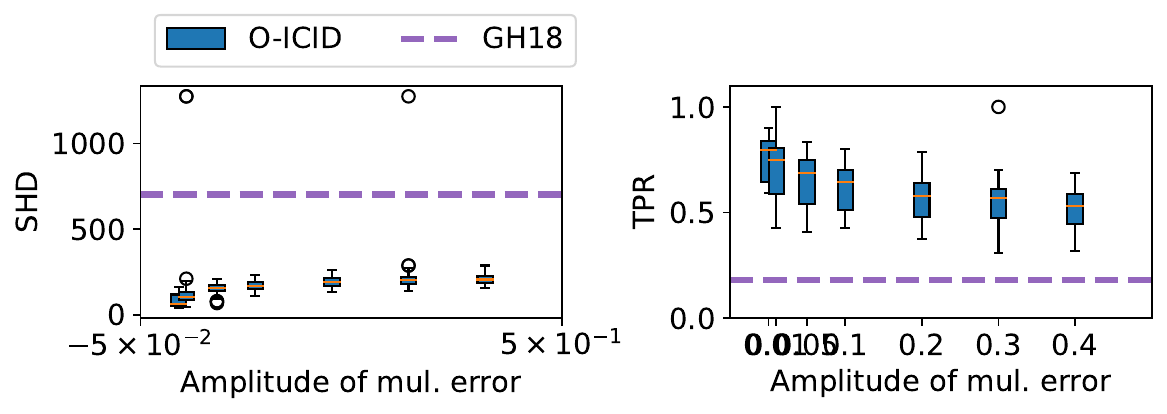}} 
    \caption{Causal structure learning by \XY. The X-axis indicates the
    amplitude $\sigma_{\epsilon}$ of average misspecifications of~$D^*$.}
    \label{fig:oicid-std-interp-Dpert-3} 
\end{figure}

\begin{figure}[H]
    \centering
    \subfigure[ER1, $\lambda=1$]{\includegraphics[width=0.490\textwidth]{figures/exp/oicid_v3exp7_sens_dstar/scat7_iset6.pdf}} 
    \subfigure[ER2, $\lambda=1$]{\includegraphics[width=0.490\textwidth]{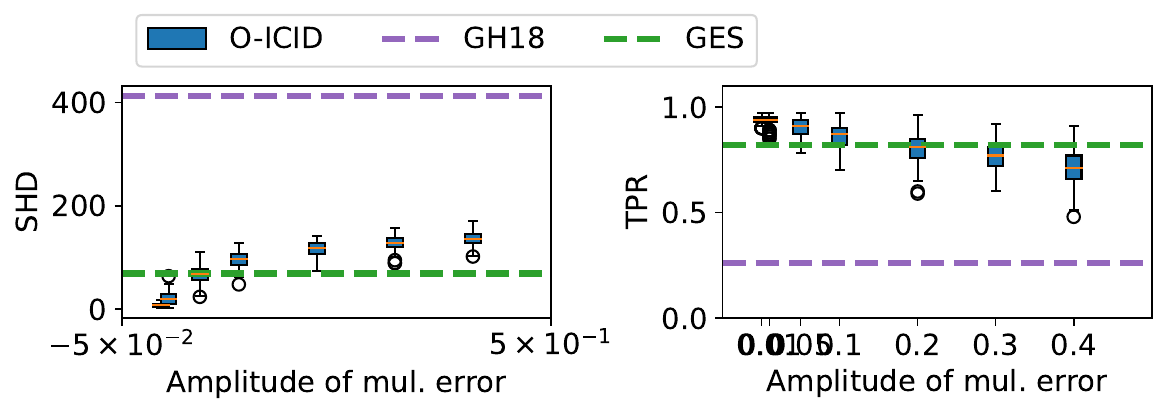}} 
    \caption{Causal structure learning by \XY. The X-axis indicates the
    amplitude $\sigma_{\epsilon}$ of average misspecifications of~$D^*$.}
    \label{fig:oicid-std-interp-Dpert-reste} 
\end{figure}

%
\bibliographystyle{plainnat}
\bibliography{refs-all}

\end{document}